\documentclass[10pt]{article} 
\usepackage[accepted]{tmlr}

\usepackage{amssymb}
\usepackage{amsmath,amssymb,amsfonts}
\usepackage{mathtools}
\usepackage{comment}
\usepackage{algorithmic}
\usepackage{graphicx}
\usepackage{makecell}
\usepackage{textcomp}
\usepackage{hyperref}
\hypersetup{hidelinks}
\usepackage{gensymb}
\usepackage{booktabs}
\usepackage{mathrsfs}
\usepackage{color, colortbl} 
\usepackage{multirow}
\usepackage{siunitx}
\usepackage{pifont}
\usepackage{subcaption}
\usepackage{amsthm}
\newtheoremstyle{boldandplain}
  {\topsep}          
  {\topsep}          
  {\it}      
  {}                 
  {\bfseries}        
  {.}                
  { }                
  {}                 

\theoremstyle{boldandplain}

\newtheorem{proposition}{Proposition}
\newtheorem{remark}{Remark}
\usepackage{xcolor}
\definecolor{DarkGreen}{RGB}{1,54,24}   
\definecolor{DarkRed}{RGB}{110,3,14}     
\usepackage{amssymb} 

\newcommand{\cmark}{\makebox[1em][l]{\ding{51}}}
\newcommand{\xmark}{\makebox[1em][l]{\ding{55}}}

\title{In-Context Residual Calibration for Uncertainty Quantification of Energy Time Series over Graphs}

\author{\name Keivan Faghih Niresi\thanks{Part of this work was carried out during Keivan Faghih Niresi's research visit at the Data, Vibration, and Uncertainty (DVU) Group, Department of Engineering, University of Cambridge.} \email keivan.faghihniresi@epfl.ch \\
      \addr Intelligent Maintenance and Operations Systems (IMOS) Laboratory \\
      EPFL, Lausanne, Switzerland
      \AND
      \name Alice Cicirello \email ac685@cam.ac.uk \\
      \addr Data, Vibration, and Uncertainty (DVU) Group, Department of Engineering \\
      University of Cambridge, United Kingdom
      \AND
      \name Olga Fink \email olga.fink@epfl.ch \\
      \addr Intelligent Maintenance and Operations Systems (IMOS) Laboratory \\
      EPFL, Lausanne, Switzerland}

\def\month{10}  
\def\year{2026} 
\def\openreview{\url{https://openreview.net/forum?id=pur0JbTRmv}} 

\begin{document}

\maketitle

\begin{abstract}
Accurate energy demand forecasting is essential for the reliable operation and planning of modern sustainable energy systems. Spatial-temporal graph neural networks (STGNNs) have recently achieved strong performance in point forecasting by jointly modeling temporal dynamics and relational dependencies across interconnected energy nodes. However, in real-world energy systems, accurate point forecasts alone are insufficient, as operators also require reliable uncertainty estimates to support risk-aware decision-making, grid stability, and operational planning under uncertainty. Conformal prediction provides a principled and model-agnostic framework for uncertainty quantification under exchangeability assumptions, making it particularly attractive for safety-critical energy applications. However, existing conformal prediction approaches often fail to fully capture the complex spatial-temporal structure of energy systems. To address these limitations, we propose STOIC (Spatial-Temporal uncertainty quantificatiOn via In-Context learning), a conformal-inspired post-hoc residual calibration framework that integrates graph-based forecasting with the in-context learning capabilities of tabular foundation models. STOIC first generates point forecasts using an STGNN and subsequently reformulates spatial-temporal residuals into a tabular representation suitable for in-context learning. Leveraging a tabular foundation model, STOIC estimates feature-conditional residual quantiles without task-specific retraining of the calibration model, effectively capturing both sequential and relational dependencies. STOIC should be viewed as a conformal-inspired residual calibration method and does not by itself provide a finite-sample distribution-free coverage guarantee under arbitrary temporal dependence. For settings where the classical exchangeability condition is appropriate, we additionally describe an optional held-out conformal correction that can recover finite-sample marginal coverage. We evaluate the basic STOIC procedure, without this optional correction, on five diverse benchmarks, including synthetic simulations as well as real-world electricity and district heating networks. Across all datasets, STOIC achieves favorable empirical coverage--sharpness trade-offs relative to the considered conformal prediction baselines for complex graph-structured energy time series.

\end{abstract}

\section{Introduction}
Accurate energy demand forecasting is becoming increasingly critical with the rapid integration of renewable energy sources, the  electrification of heating, and the growing complexity of smart energy systems \citep{hernandez2014survey, raza2018multivariate}. Reliable forecasts of electricity and heating demand are essential for a wide range of operational and planning tasks, including generation scheduling, storage management, district heating operation, demand response, and long-term infrastructure planning \citep{egging2021seasonal, hayes2015state}. As modern energy systems become more decentralized and weather-dependent, forecasting errors can propagate through interconnected infrastructures and directly affect system stability, operational efficiency, and economic cost. Over the past decade, deep learning approaches have significantly advanced forecasting performance by learning complex nonlinear temporal patterns from large historical datasets \citep{mathumitha2024intelligent,devaraj2021holistic, theiler2026integrating}. Architectures based on recurrent neural networks \citep{rahman2018predicting}, temporal convolutions \citep{jalali2021novel}, and transformer models \citep{pentsos2025hybrid, wang2022transformer} have demonstrated strong predictive capabilities across a variety of electricity and heating applications, establishing deep learning as a dominant paradigm for energy forecasting.

Despite these advances, most existing forecasting models focus primarily on temporal dynamics and often treat each location or asset independently. However, in real-world energy systems, demand patterns are inherently interconnected \citep{tascikaraoglu2016short}. Electricity consumption and heating load are influenced by spatial proximity, shared weather conditions, socio-economic factors, and correlated human activity \citep{wu2022efficient}. As a result, models that ignore spatial dependencies may fail to capture important relational information and overlook how demand at one node influences neighboring regions \citep{zhuang2023multi}. This issue is especially pronounced in district heating networks and geographically distributed electricity systems, where demand dynamics exhibit strong spatial coupling \citep{lin2021spatial}. Consequently, even highly accurate temporal models may miss a critical source of predictive information.

Graph neural networks (GNNs) provide a natural framework for modeling these dependencies \citep{corso2024graph}.  By representing energy systems as graphs, GNNs  explicitly encode  relationships between interconnected nodes and learn how information propagates across the network \citep{wu2020comprehensive, FINK2026112213}. Combined with temporal modeling components, spatial-temporal graph neural networks (STGNNs) have achieved promising results in forecasting tasks where both local temporal dynamics and network interactions are important \citep{jin2024survey}. Such models are particularly well suited to energy applications \citep{theiler2025heterogeneous, khodayar2019convolutional}, as they can jointly exploit topological structure, neighborhood influence, and system-wide interactions within a unified architecture \citep{simeunovic2021spatio, khodayar2018spatio}. However, accurate point forecasts alone are insufficient for real-world energy system operation. While STGNNs can substantially improve forecasting accuracy by exploiting spatial-temporal dependencies, point forecasts alone do not quantify the uncertainty associated with future demand. In operational settings, forecasting uncertainty is often as important as forecasting accuracy itself. Energy operators must continuously make high-stakes decisions under incomplete and evolving information, where inaccurate confidence estimates can directly affect reserve allocation, storage scheduling, market participation, and overall system reliability.

Accurate uncertainty quantification is therefore essential for risk-aware energy management \citep{daneshvar2023risk}. Underestimating uncertainty can lead to inadequate reserve allocation, inefficient storage utilization, unstable scheduling decisions, and reduced system reliability. Conversely, overly conservative estimates may result in unnecessary operational costs and inefficient resource usage.  Recent works discuss uncertainty classification from both system and modeling perspectives \citep{kamariotis2025consistent, cicirello2024physics, koune2025adversarial}, distinguishing between aleatoric uncertainty (irreducible data variability) and epistemic uncertainty (reducible lack of knowledge); when the model is misspecified, aleatoric uncertainty cannot be properly quantified, resulting in strongly coupled parameter uncertainties and model discrepancies. To handle these coupled uncertainties, classical Bayesian inference frameworks \citep{box2011bayesian, lye2021sampling}, which are implemented via sampling \citep{koune2023bayesian} or variational approaches \citep{kamariotis2025consistent, margossian2025variational}, are utilized alongside Gaussian processes \citep{cicirello2022machine, damianou2013deep, gan2020data, xiong2025deep} and Bayesian neural networks \citep{liu2023ultra, wang2023probabilistic, jospin2022hands, gal2016dropout} to provide principled probabilistic interpretations and posterior distributions over predictions. While these models are highly attractive when epistemic uncertainty is important \citep{nemani2023uncertainty}, enabling operators to assess forecast confidence and make robust decisions in weather-dependent sustainable energy systems \citep{meng2025probabilistic}, they are often hindered by restrictive distributional assumptions that fail under regime-shifting energy demand data. Furthermore, because these probabilistic methods can be computationally demanding for large-scale spatial-temporal datasets \citep{lye2022efficient, KRANNICHFELDT2026116838}, there is a growing motivation to develop distribution-free uncertainty quantification methods that remain reliable under complex spatial-temporal dependencies.

Conformal Prediction (CP) is a well-established distribution-free alternative for uncertainty quantification \citep{shafer2008tutorial}. Rather than identifying or specifying the source of uncertainty, it yields prediction intervals with coverage guarantees on a certain quantity of interest based on observed data on the same quantity. Unlike Bayesian approaches, CP does not require strong  assumptions about the underlying data-generating process \citep{angelopoulos2023conformal}. Instead, it constructs prediction intervals with finite-sample coverage guarantees under minimal assumptions, making it particularly appealing for real-world energy applications where distributional assumptions are often difficult to verify. CP also acts as a flexible wrapper that can be combined with arbitrary forecasting models to produce calibrated uncertainty estimates. However, standard conformal methods typically assume exchangeability \citep{barber2023conformal}, an assumption that is  often violated in energy systems due to temporal dependencies and spatial correlations across interconnected nodes. As a consequence,  without additional assumptions or modifications, CP methods for time-series forecasting generally cannot provide non-asymptotic distribution-free coverage guarantees under arbitrary temporal dependence and may produce intervals that are either miscalibrated or overly conservative.

Several recent methods attempt to address these limitations. Adaptive Conformal Inference (ACI) \citep{gibbs2021adaptive} and Sequential Predictive Conformal Inference (SPCI) \citep{xu2023sequential} improve calibration under temporal non‑stationarity but still treat each time series independently, ignoring relational interactions between nodes. Spatial approaches such as Localized Spatial Conformal Prediction (LSCP) \citep{jiang2024spatial} incorporate spatial structure, yet  are designed for static spatial settings and do not explicitly model temporal dependencies. More recent graph-based approaches estimate uncertainty  through relational modeling using quantile random forests with the residual history of neighboring nodes as input features \citep{jiang2025spatio},  quantile GNNs \citep{cini2025relational}, or ellipsoidal prediction sets that capture pairwise node dependencies \citep{dua2025conformal}. Nevertheless, these approaches generally require retraining or recalibration whenever the deployment setting changes. This dependency limits scalability and reduces adaptability across heterogeneous energy systems, where new networks, operating regimes, or calibration windows may frequently arise.

Tabular Foundation Models (FMs) have recently demonstrated strong in-context learning capabilities, enabling them to generalize to new structured tasks with minimal additional data \citep{awais2025foundation, liang2024foundation, goswami2024moment}. We leverage this capability for post-hoc uncertainty quantification rather than designing specialized calibration procedures for each forecasting task, as these pretrained models can perform task-specific inference without explicit retraining \citep{dong2024survey, fink2026physics}, thereby reducing calibration overhead in dynamic forecasting environments. Although the availability of dedicated FMs for spatial-temporal graph forecasting remains limited, recent successes in tabular learning with models such as TabPFN \citep{hollmann2022tabpfn, grinsztajn2025tabpfn} demonstrate strong in-context learning capabilities on structured data. This creates an opportunity to reformulate the calibration stage, rather than the full forecasting problem, into a structured tabular representation. Motivated by this perspective, spatial-temporal residuals can be organized in tabular form and calibrated using pretrained tabular foundation models, enabling efficient residual calibration without modifying the underlying forecasting architecture.

Building on this idea, we propose \textbf{S}patial-\textbf{T}emporal uncertainty quantificati\textbf{O}n via \textbf{I}n-\textbf{C}ontext learning (STOIC), a conformal-inspired framework for post-hoc uncertainty quantification in graph-based energy forecasting. STOIC first trains an STGNN to generate point forecasts over spatial-temporal energy networks. The residuals produced by the pretrained forecasting model are then transformed into a structured tabular representation that captures temporal dependencies through local residual history and rolling statistics, and relational dependencies through permutation-invariant aggregators (e.g., mean, median, min, max) over neighbors’ residuals, with the aggregation mechanism inspired by the message-passing principle. A pretrained tabular foundation model is subsequently employed to estimate feature-conditional residual quantiles and construct adaptive prediction intervals without task-specific retraining of the calibration model. The resulting intervals are adaptive because their widths depend on local temporal residual history and neighborhood residual behavior.

To evaluate the proposed framework, we conduct experiments on five datasets spanning both synthetic and real-world settings. These benchmarks include  controlled simulations, synthetic district heating networks,  real-world district heating systems, and  real-world electricity demand datasets. Across these diverse scenarios, STOIC achieves favorable empirical coverage and interval-efficiency trade-offs under a range of spatial-temporal dependencies and non-stationary dynamics.

The main contributions of this paper are summarized as follows:

\begin{enumerate}
\item We propose STOIC, a conformal-inspired post-hoc residual calibration framework that combines spatial-temporal graph forecasting with the in-context learning capabilities of tabular foundation models.
\item We develop a methodology for transforming spatial-temporal forecasting residuals into a structured tabular representation, enabling in-context learning for conditional quantile estimation.
\item We demonstrate that pretrained tabular FMs such as TabPFN can be leveraged to construct adaptive prediction intervals with favorable empirical coverage and sharpness while incorporating spatial-temporal residual dependencies without task-specific retraining.
\item We perform extensive experiments across five synthetic and real-world energy datasets, demonstrating that STOIC consistently achieves competitive or improved empirical coverage--efficiency trade-offs relative to the considered conformal prediction baselines.
\end{enumerate}

The remainder of this paper is organized as follows. Section \ref{sec:preliminaries} introduces the necessary preliminaries and background concepts. Section \ref{sec:methodology} presents the proposed STOIC framework. Sections \ref{sec:setup} and \ref{sec:results} report the experimental setup and results, respectively. Finally, Section \ref{sec:conclusion} concludes the paper and discusses future research directions.

\section{Preliminaries}
\label{sec:preliminaries}
This section introduces the foundational concepts required to understand the STOIC framework. We first establish the mathematical notation for spatial-temporal energy demand forecasting and define the objectives for both point forecasting and uncertainty quantification. We then provide an overview of CP, a distribution-free framework that motivates the residual-based uncertainty quantification perspective adopted in STOIC. Finally, we discuss the key challenges associated with applying uncertainty quantification methods to spatial-temporal graph-structured energy demand data.

\subsection{Problem Formulation and Uncertainty Quantification}

We consider a spatial-temporal energy demand forecasting problem over a network of $N$ interconnected  nodes. Let $x_{i,t}$ denote the scalar observation at node $i$ and time $t$. In real-world settings, these observations are subject to measurement noise and other sources of stochasticity, which introduce additional uncertainty into the forecasting task.  The observations across all nodes at time $t$ are represented by the vector $\mathbf{x}_t = [x_{1,t}, x_{2,t}, \dots, x_{N,t}]^\top.$ Given a historical observation window of length $W$, the forecasting model receives as input the sequence $\mathbf{X}_{t-W:t-1} = [\mathbf{x}_{t-W}, \dots, \mathbf{x}_{t-1}] \in \mathbb{R}^{N \times W},$ which captures both the temporal evolution of each node and the interactions across the network over the past $W$ time steps.
The objective is to predict the future demand vector at time $t$:
\begin{equation}
\hat{\mathbf{x}}_{t} = f_\theta(\mathbf{X}_{t-W:t-1}),
\end{equation}
where $f_\theta$ denotes a forecasting model parameterized by $\theta$. The prediction corresponding to node $i$ is given by the $i$-th component of the forecast vector:
\begin{equation}
\hat{x}_{i,t} = [\hat{\mathbf{x}}_{t}]_i.
\end{equation}

While accurate point forecasting is important, operational decision-making in energy systems also requires reliable estimates of predictive uncertainty. To this end, the goal is to construct prediction intervals (PIs) of the form:
\begin{equation}
C_{i,t}^\alpha = [\ell_{i,t}^\alpha, u_{i,t}^\alpha],
\end{equation}
where $\ell_{i,t}^{\alpha}$ and $u_{i,t}^{\alpha}$ denote the lower and upper bounds of the interval for node $i$ at time $t$, respectively, and $\alpha \in (0,1)$ is a user-specified target miscoverage level; a smaller $\alpha$ enforces higher confidence but typically results in wider prediction intervals. Ideally, the prediction intervals should satisfy the coverage property
\begin{equation}
\mathbb{P}(x_{i,t} \in C_{i,t}^\alpha) \ge 1 - \alpha,
\end{equation}
commonly referred to as \emph{marginal coverage} \citep{lei2014distribution}. A stronger requirement is \emph{conditional coverage} \citep{lei2014distribution}, which conditions on the observed historical context:
\begin{equation}
\mathbb{P}(x_{i,t} \in C_{i,t}^\alpha \mid \mathbf{X}_{t-W:t-1}) \ge 1 - \alpha.
\end{equation}

Achieving conditional coverage is substantially more challenging in spatial-temporal forecasting settings due to temporal dependencies, spatial correlations, and distributional shifts across nodes and time horizons.

In this work, we focus on post-hoc uncertainty quantification, where prediction intervals are constructed on top of an already trained forecasting model. This setting is especially attractive in practice because it avoids modifying or retraining the underlying forecasting architecture, thereby preserving the predictive capabilities of the pretrained forecaster while augmenting it with uncertainty estimates.

\subsection{Conformal Prediction}

Conformal prediction (CP) is a distribution-free framework for uncertainty quantification that constructs calibrated PIs using empirical quantiles of conformity scores. A key advantage of CP is that it can be applied as a post-hoc wrapper around arbitrary forecasting models while providing finite-sample coverage guarantees under relatively mild assumptions \citep{shafer2008tutorial}. In this work, we adopt the split conformal prediction (SCP) setting \citep{papadopoulos2002inductive}, where the available data is partitioned into disjoint training and calibration sets. The forecasting model is first trained on the training set, while the calibration set is subsequently used to estimate the distribution of prediction errors and construct calibrated uncertainty intervals.

Let $\mathcal{T}_{\mathrm{cal}}$ denote a set of $n$ calibration time indices, disjoint from the training set. For each calibration time step $t \in \mathcal{T}_{\mathrm{cal}}$, the forecasting model produces the prediction:
\begin{equation}
\hat{\mathbf{x}}_{t} = f_\theta(\mathbf{X}_{t-W:t-1}),
\end{equation}
where $f_\theta$ denotes the pretrained forecasting model. The corresponding residual vector is defined as:
\begin{equation}
\mathbf{r}_{t} = \mathbf{x}_{t} - \hat{\mathbf{x}}_{t}.
\end{equation}
The residual associated with node $i$ is therefore given by:
\begin{equation}
r_{i,t} = [\mathbf{r}_{t}]_i.
\end{equation}
Following standard SCP, the conformity score is defined as the absolute prediction residual:
\begin{equation}
s_{i,t} = \left|r_{i,t}\right|.
\end{equation}

Given a target miscoverage level $\alpha \in (0,1)$, the conformal quantile for node $i$ is estimated from the calibration conformity scores:
\begin{equation}
\hat{q}_{i}^{1-\alpha}
=
\mathrm{Quantile}_{\frac{\lceil (n+1)(1-\alpha) \rceil}{n}}
(\{s_{i,t}\}_{t \in \mathcal{T}_{\mathrm{cal}}}).
\end{equation}
The resulting PI for a new test sample at time $t^\prime$ is then constructed as:
\begin{equation}
C_{i,t^\prime}^\alpha = \left[\hat{x}_{i,t^\prime} - \hat{q}_{i}^{1-\alpha}, \; \hat{x}_{i,t^\prime} + \hat{q}_{i}^{1-\alpha}\right].
\end{equation}

Under the assumption that the calibration and test residuals are exchangeable \citep{barber2023conformal}, SCP guarantees the marginal coverage property introduced previously. This guarantee does not automatically extend to general time-series settings, where calibration and test residuals are typically dependent and non-exchangeable. Without additional assumptions or modifications, conformal prediction methods for time-series forecasting generally cannot provide non-asymptotic distribution-free coverage guarantees under arbitrary temporal dependence. Time-series conformal methods therefore often rely on additional modeling assumptions, adaptive updates, or reweighting schemes to account for temporal dependence and distribution shift. Moreover, this assumption is often violated in spatial-temporal forecasting problems due to correlations across interconnected nodes in addition to temporal dependencies.
Consequently, standard conformal approaches may produce intervals that are either miscalibrated or overly conservative. The central objective of this work is therefore to move beyond independent residual calibration and exploit the structured dependencies present in spatial-temporal energy demand data. Rather than treating calibration residuals as isolated samples, the proposed framework leverages temporal and relational information to construct more informative and adaptive prediction intervals.

\subsection{Graph Deep Learning for Point Forecasting}
\label{subsec:STGNN}
STGNNs provide a natural framework for forecasting energy demand over interconnected systems. Their fundamental principle is to jointly model temporal dynamics at individual nodes and relational dependencies across the underlying network. By combining temporal representation learning with graph-based message passing, STGNNs can capture both local consumption dynamics and cross-node interactions within a unified architecture. In STOIC, the STGNN serves as the point forecasting backbone prior to uncertainty quantification.

 We adopt a time-then-space architecture \citep{cini2025relational, gao2022equivalence}, in which temporal dependencies are first encoded independently  by a sequence encoder at each node before spatial interactions are incorporated through graph propagation layers. For each node $i$, the temporal encoder processes this historical sequence $\mathbf{x}_{i,t-W:t-1} \in \mathbb{R}^{W},$
and maps it to a latent representation:
\begin{equation}
\mathbf{h}^{(0)}_{i,t} = \mathrm{TemporalEncoder}\!\left(\mathbf{x}_{i,t-W:t-1}\right),
\end{equation}
The function $\mathrm{TemporalEncoder}(\cdot)$ can be implemented using a variety of sequence modeling architectures, including gated recurrent units (GRU), temporal convolutional networks, or transformer encoders. Aggregating the latent embeddings from all $N$ nodes yields the spatial-temporal representation:
\begin{equation}
\mathbf{H}^{(0)}_t = [\mathbf{h}^{(0)}_{1,t}, \mathbf{h}^{(0)}_{2,t}, \dots, \mathbf{h}^{(0)}_{N,t}]^\top.
\end{equation}

Following temporal encoding, the node representations are propagated through a sequence of graph neural network layers. For $\gamma= 0, \dots, \Gamma-1$, the propagation rule is defined as:
\begin{equation}
\mathbf{H}^{(\gamma+1)}_t = \mathrm{GNN}^{(\gamma)}\!\left(\mathbf{H}^{(\gamma)}_t, \mathbf{A}\right),
\end{equation}
where $\mathbf{A} \in \mathbb{R}^{N \times N}$ is the graph adjacency matrix and $\mathrm{GNN}^{(\gamma)}(\cdot)$ denotes the $\gamma$-th message-passing layer, such as diffusion graph convolution \citep{atwood2016diffusion, li2018diffusion}. This message-passing mechanism enables each node to aggregate information from its spatial neighborhood, allowing the learned representations to capture both  node-specific temporal behavior and interactions among connected nodes. The final latent representation $\mathbf{H}^{(\Gamma)}_t$ therefore provides a compact encoding of the spatial-temporal system state and forms the basis for downstream point forecasting. In the following section, these graph-based forecasts are combined with residual-based in-context uncertainty estimation to construct adaptive prediction intervals for spatial-temporal energy demand forecasting.

\section{Proposed Method}
\label{sec:methodology}
The STOIC framework decouples point forecasting from uncertainty quantification to address the structured  residual behavior commonly observed in networked energy systems. Rather than relying on static calibration schemes, STOIC reformulates uncertainty estimation as an in-context learning problem by leveraging the zero-shot generalization capabilities of tabular FMs. The framework first employs an STGNN architecture to capture spatial-temporal dependencies and generate point forecasts. The resulting prediction residuals are then transformed into a structured tabular representation that captures both local neighborhood interactions and temporal dynamics. Finally, this contextual representation is used to construct adaptive PIs without requiring task-specific retraining of the calibration model. Throughout the paper, STOIC refers to the proposed post-hoc uncertainty quantification framework, where TabPFN is used for adaptive residual quantile estimation. It uses in-context learning on residual features to construct prediction intervals whose widths depend on local sequential and relational context. Since STOIC relies on the quality of the TabPFN residual quantile estimates, it should be interpreted as a conformal-inspired residual calibration method rather than as a finite-sample-valid conformal method under arbitrary temporal dependence. The following subsections detail the graph-based point forecasting architecture, the residual feature engineering procedure, and the in-context residual calibration process. An overview of the full pipeline is illustrated in Fig.~\ref{overal}.

\begin{figure*} \centering \includegraphics[width=\linewidth]{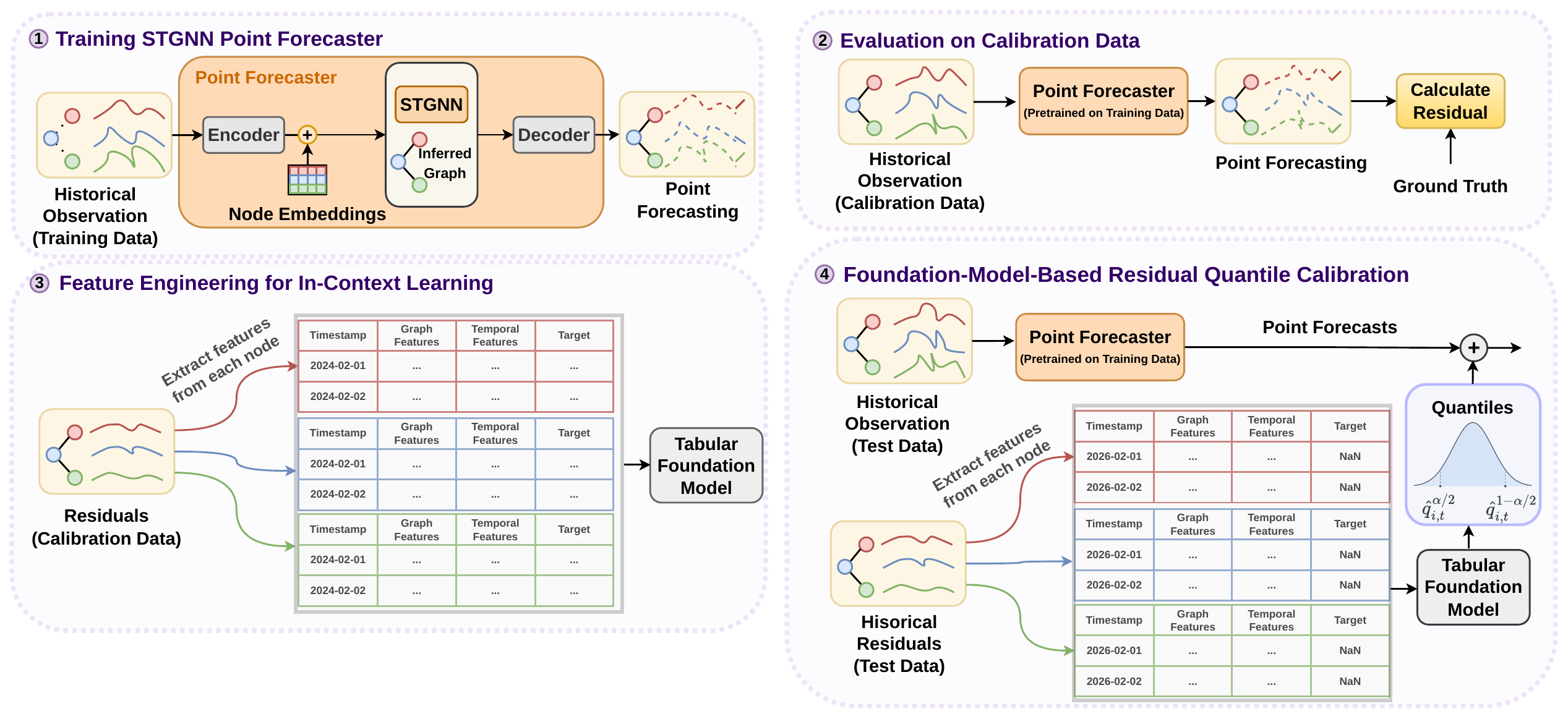} \caption{Overview of the proposed STOIC framework. The pipeline consists of four stages: (1) training an STGNN-based point forecasting backbone  on historical spatial-temporal  observations, (2) computing calibration  residuals using the pretrained forecaster, (3) transforming residuals into structured tabular representations  that capture  temporal dynamics  and spatial neighborhood interactions, and (4) performing in-context calibration with a tabular FM to estimate conditional residual  quantiles and construct PIs.} \label{overal} 
\end{figure*}

\subsection{Spatial-Temporal Architecture for Point Forecasting}

The forecasting backbone of STOIC follows a shared spatial-temporal architecture. For each node $i$, the historical input sequence is first projected into a latent feature space through a shared encoder:
\begin{equation}
\mathbf{h}^{(0)}_{i,t} = \mathrm{Encoder}(\mathbf{x}_{i,t-W:t-1}).
\end{equation}
The resulting node-level embeddings are  stacked to form the global latent representation:
\begin{equation}
\mathbf{H}^{(0)}_t = [\mathbf{h}^{(0)}_{1,t}, \dots, \mathbf{h}^{(0)}_{N,t}]^\top \in \mathbb{R}^{N \times d}.
\end{equation}
To incorporate node-specific characteristics, a learnable node embedding matrix  \(\mathbf{V} \in \mathbb{R}^{N \times d}\) is  added to the latent representation $\mathbf{H}^{(0)}_t \leftarrow \mathbf{H}^{(0)}_t + \mathbf{V}$.

The resulting latent representations are subsequently processed by the STGNN backbone introduced in Section \ref{subsec:STGNN}, which jointly models temporal dynamics and inter-node interactions:
\begin{equation}
\mathbf{Z}_t = \mathrm{STGNN}(\mathbf{H}^{(0)}_t, \mathbf{A}).
\end{equation}
Finally, a shared decoder maps the latent node representations  to the next-step forecasts:
\begin{equation}
\hat{\mathbf{x}}_{t} = \mathrm{Decoder}(\mathbf{Z}_t),
\end{equation}
where $\hat{\mathbf{x}}_{t} \in \mathbb{R}^{N}$ denotes  the predicted observation vector at time $t$. These point forecasts are subsequently used in the residual uncertainty estimation stage to construct adaptive prediction intervals.

\subsection{Learning Sparse Graph from Data}
In many real-world energy forecasting scenarios, the underlying interaction graph is not available a priori. Rather than assuming access to a predefined physical topology, STOIC learns the spatial dependencies directly from data. This design makes the framework applicable to systems where the true connectivity is unavailable, incomplete, or only partially reliable. Accordingly, the graph is treated as a latent object whose structure is learned jointly with the forecasting model. Given a collection of $N$ nodes, we parameterize the pairwise connectivity through a learnable score matrix $\mathbf{\Phi} \in \mathbb{R}^{N \times N}$, where each entry \(\phi_{ij}\) represents the affinity between nodes \(i\) and \(j\).

To construct a sparse neighborhood, we sample $K$ neighbors
without replacement for each node $i$ from its underlying categorical distribution
$\boldsymbol{\pi}_i$, whose entries are defined as
\begin{equation}
\pi_{ij} =
\frac{\exp(\phi_{ij})}{\sum_{k=1}^{N} \exp(\phi_{ik})},
\quad j \in \{1, \dots, N\}.
\end{equation}
The resulting sampled neighborhood defines the local connectivity set
$\mathcal{N}(i)$ for each node $i$. The graph is learned end-to-end using
Gumbel-Top-$K$ sampling \citep{kool2020ancestral} together with a
straight-through gradient estimator \citep{bengio2013estimating}. This allows
gradients to propagate through the discrete neighborhood selection process
during training. As a result, STOIC can automatically discover sparse,
task-specific spatial structures that are optimized for forecasting
performance, even when the true system topology is unknown.

\subsection{Spatial-Temporal Feature Engineering for In-Context Learning}

After training the point forecasting backbone, STOIC is applied to the calibration split to obtain node-level prediction residuals. For every calibration time step $t \in \mathcal{T}_{\mathrm{cal}}$, the residual vector is computed as:
\begin{equation}
\mathbf{r}_t = \mathbf{x}_t - \hat{\mathbf{x}}_t,
\end{equation}
where \(\mathbf{x}_t\) denotes the observations and \(\hat{\mathbf{x}}_t\) represents the corresponding point forecasts produced by the STGNN backbone.

The objective of the subsequent feature engineering stage is to transform these residuals into structured tabular representations that can serve as contextual inputs for the tabular FM. Since tabular foundation models are designed to operate on fixed tabular inputs rather than graph-structured data, the spatial dependencies encoded in the underlying network cannot be directly processed in their native form. STOIC therefore converts the temporal and spatial residual structure into engineered tabular features that preserve informative neighborhood interactions while remaining compatible with the tabular FM architecture. Rather than treating calibration residuals as independent samples drawn from a stationary distribution, STOIC explicitly models their temporal and spatial structure. This design is motivated by empirical observations that forecasting errors in energy systems often exhibit strong temporal autocorrelation, heteroscedasticity, and cross-node dependence \citep{huang2023uncertainty, jia2020residual}. Consequently, the uncertainty associated with a given node may depend not only on its own recent error history but also on the residual behavior of neighboring nodes.

To capture these dependencies, STOIC constructs feature vectors using only historical residual information available prior to the prediction time step, thereby preventing information leakage during calibration. For each target node \(i\), the resulting feature representation encodes multiple sources of contextual information derived from the calibration residual history. Specifically, the features are organized into three categories:
\begin{itemize}
    \item Temporal features, which capture the recent residual dynamics of the target node;
    
    \item Graph-based features, which summarize residual patterns from neighboring nodes in the learned graph structure; and
    
    \item Combined features, which integrate temporal and spatial context into a unified representation for in-context calibration.
\end{itemize}

These engineered representations provide the tabular FM with structured contextual information that enables adaptive estimation of the conditional residual distribution and, ultimately, the construction of context-aware prediction intervals.

\subsubsection{Temporal Features}

To capture the local temporal dynamics of forecasting errors at node $i$, STOIC constructs features from the historical residual sequence $\{r_{i,t-k} : k \ge 1\}$. These features are designed to summarize recent error behavior, including persistence, magnitude, and local variability, which are informative for estimating the conditional uncertainty at the current time step. Given a lag set $\mathcal{P} = \{1, \dots, P\}$ and a collection of rolling window sizes $\mathcal{W}$, summary statistics over past residuals are computed. For each window \(w \in \mathcal{W}\), the rolling mean, rolling median, and rolling standard deviation are defined as:
\begin{equation}
\mu^{(w)}_{i,t} = \frac{1}{w} \sum_{k=1}^{w} r_{i,t-k}, \quad 
\nu^{(w)}_{i,t} = \mathrm{median}\left(\{r_{i,t-k}\}_{k=1}^{w}\right), \quad
\delta^{(w)}_{i,t}
=
\sqrt{
\frac{1}{w-1}
\sum_{k=1}^{w}
\left(r_{i,t-k}-\mu^{(w)}_{i,t}\right)^2
}.
\end{equation}
The temporal feature vector is then constructed as:
\begin{equation}
\mathbf{f}^{\mathrm{temp}}_{i,t}
=
\left[
\{r_{i,t-p}\}_{p\in\mathcal{P}},
\{\mu^{(w)}_{i,t},
\nu^{(w)}_{i,t},
\delta^{(w)}_{i,t}\}_{w\in\mathcal{W}}
\right].
\end{equation}
These features provide the tabular FM with information about short-term residual trends and local uncertainty regimes. In particular, they enable the model to distinguish between stable forecasting periods and periods characterized by rapidly changing or highly variable residual behavior. 

\subsubsection{Graph-based Features}

In addition to temporal patterns, STOIC explicitly incorporates information from neighboring nodes. Since tabular foundation models cannot directly process graph-structured inputs or perform message passing over adjacency relations, the spatial structure of the network must first be transformed into fixed-dimensional tabular representations. To achieve this, STOIC summarizes neighborhood residual behavior  through permutation-invariant aggregation operators, effectively translating graph-level information into a format compatible with tabular FMs.

This design mirrors the message-passing mechanism used in graph neural networks, where node representations are updated through aggregation of information from local neighborhoods. Following the same principle, STOIC constructs graph-based features that summarize neighboring residual dynamics while remaining invariant to the ordering of neighboring nodes.

Let $\mathcal{N}(i)$ denote the set of neighbors of node $i$ in the learned or predefined adjacency structure $\mathbf{A}$. For each lag $p \in \mathcal{P}$ and  aggregation operator $g(\cdot)$, we compute:
\begin{equation}
a^{g,p}_{i,t}
=
g\left(\left\{ r_{j,t-p} : j \in \mathcal{N}(i) \right\}\right),
\end{equation}
where $g(\cdot)$ is a permutation-invariant operator such as the mean, median, minimum, or maximum.

The resulting graph-based feature vector is defined as:
\begin{equation}
\mathbf{f}^{\mathrm{graph}}_{i,t}
=
\big[
a^{\mathrm{mean},p}_{i,t},
a^{\mathrm{median},p}_{i,t},
a^{\mathrm{min},p}_{i,t},
a^{\mathrm{max},p}_{i,t}
\big]_{p=1}^{P}.
\end{equation}

These features encode the local spatial context of  node $i$ by summarizing how neighboring nodes behaved in the recent past. This is particularly important in energy systems, where nearby or functionally related locations often exhibit correlated demand patterns, shared environmental influences, or common system-level disturbances. Consequently, residual behavior observed at neighboring nodes can provide informative signals for estimating uncertainty at the target nodes.

\subsubsection{Combined Feature Representation}

The final contextual representation for node $i$ at calibration time $t$ is obtained by concatenating the temporal and graph-based feature components:
\begin{equation}
\mathbf{f}_{i,t}
=
\left[
\mathbf{f}^{\mathrm{temp}}_{i,t},
\mathbf{f}^{\mathrm{graph}}_{i,t}
\right].
\end{equation}

This concatenated vector defines a structured tabular sample whose target variable is the corresponding calibration residual  \(r_{i,t}\). In STOIC, these feature-target pairs constitute the contextual dataset provided to the tabular FM used for in-context calibration. Feature extraction is performed independently for each node, resulting in a collection of node-specific tabular datasets that encode both temporal residual dynamics and spatial neighborhood information. This representation preserves local temporal history, incorporates spatial context through permutation-invariant neighbor aggregation, and avoids leakage by relying only on past residuals. Together, these features provide the foundation for the next stage, where the tabular FM (e.g., TabPFN) is used to infer uncertainty estimates.

\subsection{Foundation-Model-Based Residual Quantile Calibration}

Given the engineered feature representations, STOIC performs uncertainty estimation by formulating residual modeling as a tabular regression problem. Rather than relying on fixed empirical residual quantiles as in standard CP, STOIC leverages a tabular FM to learn the conditional residual distribution in an in-context manner. For each node $i$, we construct a dataset consisting of feature-residual pairs $\{(\mathbf{f}_{i,t}, r_{i,t})\}_{t \in \mathcal{T}_{\mathrm{cal}}}$, where \(\mathbf{f}_{i,t}\) denotes the engineered contextual feature vector and \(r_{i,t}\) is the corresponding calibration residual. A pretrained tabular FM is then applied in-context to map features, constructed from the history of past residuals and their spatial context, to residual values:

\begin{equation}
r_{i,t} = h_\psi(\mathbf{f}_{i,t}),
\end{equation}
where $h_\psi(\cdot)$ denotes the tabular FM parameterized by $\psi$.

At inference time, given a new input window $\mathbf{X}_{t-W:t-1}$  from the test set, the STGNN backbone first produces the point forecast $\hat{x}_{i,t}$. The corresponding contextual feature vector $\mathbf{f}_{i,t}$ is then constructed using the same feature engineering procedure applied during calibration. Conditioned on this representation, the tabular FM predicts quantiles of the residual distribution:
\begin{equation}
\hat{q}_{i,t}^{\alpha} = h_\psi^{(\alpha)}(\mathbf{f}_{i,t}),
\end{equation}
where $h_\psi^{(\alpha)}(\cdot)$ denotes the predicted $\alpha$-quantile of the conditional residual distribution. The PI is subsequently obtained by shifting the predicted residual quantiles around the point forecast:
\begin{equation}
C_{i,t}^\alpha =
\left[
\hat{x}_{i,t} + \hat{q}_{i,t}^{\alpha/2}, \;
\hat{x}_{i,t} + \hat{q}_{i,t}^{1-\alpha/2}
\right].
\end{equation}

This formulation can be interpreted as a conformal-inspired, feature-conditional residual calibration approach. While classical conformal methods use global residual quantiles, STOIC adapts the interval width based on the local temporal and spatial context encoded in $\mathbf{f}_{i,t}$. As a result, the method produces heteroscedastic and context‑aware PIs. Importantly, this procedure does not require retraining the forecasting backbone. The tabular FM operates purely on residual features, enabling a modular pipeline where forecasting and uncertainty quantification are decoupled. This also allows leveraging the zero‑shot generalization capabilities of modern tabular models to perform calibration with limited data.

\begin{remark}[Distinction with quantile regression]
We emphasize that STOIC differs from conventional task-specific quantile regression approaches in how the quantile estimator is used and conditioned. Classical quantile regression directly models conditional quantiles of the target variable and therefore typically requires either multiple independently trained models or dedicated output heads for different quantile levels.  As a result, uncertainty estimation becomes tightly coupled to the forecasting architecture and often requires extensive task-specific retraining, hyperparameter tuning, and calibration. In contrast, STOIC performs uncertainty estimation post hoc on the residuals of an arbitrary pretrained point forecaster. The framework transforms spatial-temporal residual structure into a tabular contextual representation and uses a pretrained tabular FM to infer conditional residual quantiles in-context. This design yields several important advantages. First, STOIC is fully modular and can be applied to any forecasting architecture without altering or retraining the underlying predictor. Second, uncertainty estimation is conditioned on structured residual context rather than learned jointly with the forecasting objective, enabling adaptive interval construction even when the forecasting model itself was not optimized for probabilistic prediction. Third, by leveraging pretrained tabular FMs, STOIC inherits their zero-shot capabilities and can operate effectively with limited calibration data without fitting a task-specific quantile estimator from scratch.

Conceptually, STOIC remains aligned with the residual-calibration perspective common in conformal and conformal-inspired methods, where uncertainty is estimated through residual behavior rather than direct target quantile prediction. However, STOIC extends this principle beyond global calibration by introducing feature-conditional residual modeling tailored to graph-structured spatial-temporal systems.
\end{remark}

\begin{remark}[Expressive capacity and motivation for the feature representation]
\label{rem:capacity}
Let $\mathbf{f}_{i,t}$ denote the engineered feature vector constructed from the recent residual history of node $i$ together with aggregated permutation‑invariant statistics of neighboring residuals.
Any continuous function defined on a compact domain that depends on these quantities can, in principle, be approximated arbitrarily well by a universal approximator operating on $\mathbf{f}_{i,t}$.
Consequently, an idealized model conditioned on the engineered features can emulate a restricted class of one‑hop message‑passing functions whose neighborhood interactions are fully characterized by the selected aggregation operators.

Furthermore, if the true conditional distribution of the residual $r_{i,t}$ depends on the neighbors' past residuals solely through these permutation-invariant statistics, then $\mathbf{f}_{i,t}$ captures all the conditioning information necessary to determine the true conditional quantiles. Even when this assumption is only partially satisfied, the selected aggregates can still provide a compact, informative summary for conditional quantile estimation by the tabular model. In practice, the fixed pretrained model $h_\psi(\cdot)$ produces a non-linear empirical approximation, the fidelity of which is governed by the model's inherent capacity and the relevance of its pretraining distribution.
\end{remark}

\begin{proposition}[Asymptotic conditional coverage under an oracle estimator]
\label{prop:oracle_cal}
Let $r_{i,t}=x_{i,t}-\hat{x}_{i,t}$ and let
$F_{r|f}(\cdot\mid\mathbf{f}_{i,t})$ be its true conditional cumulative
distribution function. For any $\beta\in(0,1)$, denote by
$q^{\beta}(\mathbf{f}_{i,t})$ the true conditional quantile satisfying
$\mathbb{P}\bigl(r_{i,t}\le q^{\beta}(\mathbf{f}_{i,t})\mid\mathbf{f}_{i,t}\bigr)=\beta$.
Assume that:
\begin{enumerate}
\item $F_{r|f}(\cdot\mid\mathbf{f}_{i,t})$ is continuous and strictly monotonically increasing;
    \item an \emph{oracle} estimator provides consistent estimates of both
    conditional quantiles:
    \[
    \hat{q}^{\alpha/2}(\mathbf{f}_{i,t})\xrightarrow{p} q^{\alpha/2}(\mathbf{f}_{i,t}),
    \quad
    \hat{q}^{1-\alpha/2}(\mathbf{f}_{i,t})\xrightarrow{p} q^{1-\alpha/2}(\mathbf{f}_{i,t}),
    \]
    as the number of calibration samples $n\to\infty$.
\end{enumerate}
Then the prediction interval
\[
C_{i,t}^{\alpha}= \Bigl[\,\hat{x}_{i,t}+\hat{q}^{\alpha/2}(\mathbf{f}_{i,t}),\;
                     \hat{x}_{i,t}+\hat{q}^{1-\alpha/2}(\mathbf{f}_{i,t})\,\Bigr]
\]
achieves asymptotic conditional coverage:
\[
\mathbb{P}\bigl(x_{i,t}\in C_{i,t}^{\alpha}\mid\mathbf{f}_{i,t}\bigr)
\;\xrightarrow{p}\; 1-\alpha .
\]
\end{proposition}

\begin{proof}
Because \(x_{i,t}=\hat{x}_{i,t}+r_{i,t}\), the event
\(x_{i,t}\in C_{i,t}^{\alpha}\) is equivalent to
\[
\hat{q}^{\alpha/2}(\mathbf{f}_{i,t})
\le r_{i,t}
\le
\hat{q}^{1-\alpha/2}(\mathbf{f}_{i,t}).
\]
It therefore suffices to establish coverage for the residual interval.

\begin{equation}
\mathbb{P}\bigl(r_{i,t}\le q^{\alpha/2}(\mathbf{f}_{i,t})\mid\mathbf{f}_{i,t}\bigr)
= \alpha/2,
\end{equation}
\begin{equation}
\mathbb{P}\bigl(r_{i,t}\le q^{1-\alpha/2}(\mathbf{f}_{i,t})\mid\mathbf{f}_{i,t}\bigr)
= 1-\alpha/2 .
\end{equation}

Under the oracle assumption, the estimated quantiles are consistent:
\begin{align}
\hat{q}^{\alpha/2}(\mathbf{f}_{i,t}) &\xrightarrow{p} q^{\alpha/2}(\mathbf{f}_{i,t}),\\
\hat{q}^{1-\alpha/2}(\mathbf{f}_{i,t}) &\xrightarrow{p} q^{1-\alpha/2}(\mathbf{f}_{i,t}).
\end{align}

Since $F_{r|f}(\cdot\mid\mathbf{f}_{i,t})$ is continuous at the true
quantiles, the Continuous Mapping Theorem yields
\begin{equation}
F_{r|f}\bigl(\hat{q}^{\alpha/2}(\mathbf{f}_{i,t})\mid\mathbf{f}_{i,t}\bigr)
\xrightarrow{p} F_{r|f}\bigl(q^{\alpha/2}(\mathbf{f}_{i,t})\mid\mathbf{f}_{i,t}\bigr)
= \alpha/2,
\end{equation}
\begin{equation}
F_{r|f}\bigl(\hat{q}^{1-\alpha/2}(\mathbf{f}_{i,t})\mid\mathbf{f}_{i,t}\bigr)
\xrightarrow{p} F_{r|f}\bigl(q^{1-\alpha/2}(\mathbf{f}_{i,t})\mid\mathbf{f}_{i,t}\bigr)
= 1-\alpha/2 .
\end{equation}

Therefore, by consistency of the estimated quantiles and continuity of
$F_{r|f}(\cdot\mid\mathbf{f}_{i,t})$, the probability content of the
estimated residual interval converges to that of the corresponding
oracle interval:
\[
\begin{aligned}
&\mathbb{P}\bigl(
\hat{q}^{\alpha/2}(\mathbf{f}_{i,t})
\le r_{i,t}
\le
\hat{q}^{1-\alpha/2}(\mathbf{f}_{i,t})
\mid \mathbf{f}_{i,t}
\bigr) \xrightarrow{p}
F_{r|f}\bigl(
q^{1-\alpha/2}(\mathbf{f}_{i,t})
\mid\mathbf{f}_{i,t}
\bigr)
-
F_{r|f}\bigl(
q^{\alpha/2}(\mathbf{f}_{i,t})
\mid\mathbf{f}_{i,t}
\bigr)
\\
&\qquad=
(1-\alpha/2)-\alpha/2
=
1-\alpha .
\end{aligned}
\]

Shifting the interval by the point forecast $\hat{x}_{i,t}$ preserves the
coverage probability, so the same limit holds for $x_{i,t}$.
\end{proof}

\begin{remark}[On the oracle assumption and the choice of TabPFN] 
\label{rem:oracle_gap}
Proposition~\ref{prop:oracle_cal} establishes an asymptotic oracle result under the assumption that the conditional residual quantiles are estimated consistently. Existing approaches~\citep{jiang2025spatio, jiang2024spatial, xu2023sequential} satisfy this requirement using quantile random forests trained from scratch on the calibration set, which are provably consistent under certain assumptions~\citep{xu2023sequential, meinshausen2006quantile}. In contrast, TabPFN is a prior-fitted network~\citep{hollmann2022tabpfn} designed to approximate Bayesian inference. Because it relies on a synthetic prior, it lacks theoretical guarantees for consistently estimating arbitrary conditional distributions~\citep{nagler2023statistical}. Despite this theoretical gap, TabPFN introduces practical advantages that make it highly effective for the STOIC framework. First, its in-context learning capabilities enable zero-shot calibration across diverse nodes without the need for task-specific retraining. Second, its pretrained inductive biases support flexible residual modeling in the low-to-moderate calibration-data regimes considered in this work. Finally, TabPFN’s in‑context learning capability makes it straightforward to share one pretrained model across all nodes, simplifying deployment compared to methods that require per‑node training.

Overall, TabPFN should be viewed as an empirical approximation to the oracle quantile estimator considered in Proposition~\ref{prop:oracle_cal}. Since TabPFN does not have a general consistency guarantee for the conditional residual quantiles, Proposition~\ref{prop:oracle_cal} does not constitute an asymptotic coverage guarantee for the implemented TabPFN-based STOIC method. Rather, it characterizes the target asymptotic behavior that would follow under consistent conditional quantile estimation. Remark~\ref{rem:tv_bound} provides a qualitative interpretation of how distributional approximation error can affect coverage.
\end{remark}

\begin{remark}[Coverage bound without the oracle assumption]
\label{rem:tv_bound}

Proposition~\ref{prop:oracle_cal} assumes a consistent estimator of the conditional quantiles. In practice, the pretrained tabular foundation model only provides an approximation to the true conditional residual distribution. The following argument gives a qualitative interpretation of how approximation error affects coverage. For a fixed calibration set, let $\hat{q}^{\alpha/2}(\mathbf{f}_{i,t}),\hat{q}^{1-\alpha/2}(\mathbf{f}_{i,t})$ denote the quantiles predicted by the tabular FM. Let $P_{\psi}^c$ denote the conditional predictive residual distribution induced by the tabular FM, whose $\alpha/2$ and $1-\alpha/2$ quantiles equal these values, and assume that it is continuous at these quantiles. By construction,
$$
\mathbb{P}_{P_{\psi}^c}\!\bigl( r_{i,t} \in [\hat{q}^{\alpha/2},\,
\hat{q}^{1-\alpha/2}] \bigr) = 1-\alpha .
$$
Let $P_{t}^c$ be the true conditional residual distribution, and let $B = \bigl\{ r_{i,t} \in [\hat{q}^{\alpha/2},\,\hat{q}^{1-\alpha/2}] \bigr\}$. By the definition of total variation distance,
$$
\bigl|\mathbb{P}_{P_{t}^c}(B) - \mathbb{P}_{P_{\psi}^c}(B)\bigr|
\le \mathrm{TV}(P_{t}^c,\,P_{\psi}^c) .
$$
Consequently,
\begin{equation}
\mathbb{P}_{P_{t}^c}(B)
 \ge  1-\alpha - \mathrm{TV}\bigl(P_{t}^c,P_{\psi}^c\bigr),
\end{equation}
where $\mathrm{TV}(\cdot,\cdot)$ denotes the total variation distance \citep{cini2025relational}. This is a lower bound: actual coverage can be higher than $1-\alpha$ if the interval is conservatively wide, even when $\mathrm{TV}>0$. The bound is not computable in practice, but it gives a qualitative interpretation: the better $P_{\psi}^c$ approximates the true conditional residual distribution, the smaller the potential coverage gap. Since this total variation distance is neither known nor controlled for TabPFN in our setting, this bound does not provide a practical coverage guarantee for the implemented STOIC method.
\end{remark}

\subsection{Computational Complexity}
\label{sec:computational_complexity}

STOIC introduces two post-hoc operations after the point forecaster has been trained: spatial-temporal residual feature construction and in-context residual quantile estimation with TabPFN. Let $\mathcal{T}_{\mathrm{test}}$ denote the set of test indices, $n_{\mathrm{test}}=|\mathcal{T}_{\mathrm{test}}|$, and recall that $n=|\mathcal{T}_{\mathrm{cal}}|$, $\mathcal{P}=\{1,\ldots,P\}$ is the lag set, $\mathcal{W}$ is the set of rolling-window sizes, and $K=|\mathcal{N}(i)|$ is the number of neighbors per node. For each node and time step, STOIC computes $P$ lagged residuals, rolling statistics for every $w\in\mathcal{W}$, and four graph-based statistics (mean, median, minimum, and maximum) for every lag. Under a direct implementation, the approximate feature-construction cost is $\mathcal{O}\!\left(N(n+n_{\mathrm{test}})\left[P+\sum_{w\in\mathcal{W}}w+PK\right]\right)$, where constant factors associated with the rolling and neighborhood operators are omitted, and the precise cost of median computation depends on its implementation. For fixed $P$, $\mathcal{W}$, and $K$, this cost grows approximately linearly with the number of nodes and processed time steps, whereas for a fully connected graph, the graph-based component becomes quadratic in $N$.  For each node, the calibration pairs $\{(\mathbf{f}_{i,t},r_{i,t})\}_{t\in\mathcal{T}_{\mathrm{cal}}}$ are supplied to the pretrained TabPFN as labeled in-context examples, while the test features serve as query examples; no task-specific gradient-based optimization is performed. Letting $M=n+n_{\mathrm{test}}$ and $D_f$ be the feature dimension, the TabPFN architecture is reported to have computational complexity $\mathcal{O}(M^2+D_f^2)$ and memory complexity $\mathcal{O}(MD_f)$ for a table with $M$ samples and $D_f$ features \citep{grinsztajn2025tabpfn, hollmann2022tabpfn}. Since STOIC applies this procedure independently to all $N$ nodes, a conservative architecture-level bound is $\mathcal{O}\bigl(N(M^2 + D_f^2)\bigr)$. TabPFN can additionally cache the labeled calibration context for reuse across test queries. Overall, STOIC requires more calibration-time computation than methods based on a single empirical residual quantile, but it avoids task-specific gradient-based training of the calibration model.

\subsection{Discussion on Finite-Sample Coverage}

Standard CP methods provide finite‑sample marginal coverage under exchangeability between calibration and test examples. However, in time‑series forecasting, calibration and test residuals are typically temporally dependent, non‑stationary, and subject to regime changes; in graph‑structured settings, cross‑node dependence adds further complexity. Therefore, without additional assumptions or modifications, conformal methods for time series generally cannot offer the same non‑asymptotic distribution‑free guarantee as classical split conformal prediction. This limitation is not specific to STOIC. Existing time‑series conformal methods introduce extra mechanisms to handle dependence. For instance, some methods model future residual quantiles from past residuals and establish asymptotic conditional coverage under quantile‑regression consistency \citep{jiang2025spatio, xu2023sequential}; others reweight conformity scores and obtain asymptotic conditional coverage under their own assumptions \citep{neglia2025rescp} or update calibration quantities over time to account for distribution shift or sequential dependence \citep{gibbs2021adaptive}.

STOIC adopts the same model‑based residual calibration philosophy. It constructs temporal and relational residual features and estimates feature‑conditional residual quantiles through a tabular foundation model. Consequently, STOIC should be viewed as a conformal‑inspired adaptive residual calibration method, not as one that provides finite‑sample distribution‑free coverage under arbitrary temporal dependence. The oracle result in Proposition~\ref{prop:oracle_cal} formalizes the targeted property: if the conditional residual quantiles are consistently estimated, the resulting intervals achieve asymptotic conditional coverage. In practice, since TabPFN does not provide a general consistency guarantee for arbitrary conditional distributions, STOIC's coverage is empirically assessed at the average, node-wise, and regime-wise levels.

If finite‑sample marginal coverage under the classical exchangeability condition is desired, STOIC can be augmented with a held‑out conformal correction. Specifically, the calibration period is split into two disjoint parts: one subset serves as the in‑context residual modeling set for TabPFN, and the second, held‑out subset $\mathcal{T}_{\mathrm{conf}}$ is reserved exclusively for conformal calibration. For each held‑out calibration example, let $\ell_{i,t}$ and $u_{i,t}$ denote the preliminary STOIC bounds, and define the conformity score as
\begin{equation}
s_{i,t} = \max\{\ell_{i,t} - x_{i,t},\; x_{i,t} - u_{i,t},\; 0\}.
\end{equation}
For each node $i$, let $\hat{s}_{i,1-\alpha}$ be the empirical
$\frac{\left\lceil (|\mathcal{T}_{\mathrm{conf}}|+1)(1-\alpha)\right\rceil}
{|\mathcal{T}_{\mathrm{conf}}|}$-quantile
of the scores $\{s_{i,t}\}_{t\in\mathcal{T}_{\mathrm{conf}}}$. The corrected interval is then
\begin{equation}
C_{i,t}^\alpha =
\bigl[\ell_{i,t} - \hat{s}_{i,1-\alpha},\;
u_{i,t} + \hat{s}_{i,1-\alpha}\bigr].
\end{equation}
Under exchangeability between the held‑out conformal calibration examples and the test examples for each node, this interval recovers the standard finite‑sample marginal coverage guarantee. Importantly, this correction does not remove the adaptivity of STOIC: the preliminary bounds $\ell_{i,t}$ and $u_{i,t}$ are still produced by TabPFN from temporal and relational residual features; only a node-specific safety margin is added. If the preliminary STOIC intervals already exhibit small conformity scores on the held-out subset, the correction is correspondingly small. Otherwise, the correction expands the preliminary intervals according to the empirical conformity quantile. Under exchangeability, the corrected intervals achieve finite-sample marginal coverage of at least \(1-\alpha\). The cost of this optional correction is statistical rather than computational: the original calibration period must be divided into an in‑context modeling subset and a held‑out conformal subset, leaving fewer examples for TabPFN's in‑context residual estimation. The experimental results reported in this paper correspond to basic STOIC; the optional held-out conformal correction described above is not applied in the main experiments. Using TabPFN on the whole calibration set, basic STOIC achieved aggregate empirical coverage close to or above the nominal level across the evaluated datasets.

\section{Experimental Setup}
\label{sec:setup}

\subsection{Datasets}
We evaluate STOIC on a diverse collection of synthetic and real-world datasets, including controlled simulations and large-scale energy systems. Across all datasets, we define the forecasting task as predicting one-day-ahead demand from the previous seven days of observations. To facilitate model training and evaluation, each dataset is partitioned chronologically into training, calibration, and testing sets, following a 70\%, 15\%, and 15\% split, respectively. The training split is used exclusively for learning the point forecasting backbone, while the calibration split is reserved for residual-based uncertainty estimation. Furthermore, for training the point forecaster, all data are standardized by subtracting the mean and dividing by the standard deviation, with statistics computed solely from the training set to prevent data leakage.

\noindent \textbf{Simulated Control Signals (SCS):} To evaluate STOIC under controlled conditions with known ground-truth dynamics, we generate a synthetic dataset driven by two latent control signals, $u_1(t)$ and $u_2(t)$, modeled as sinusoids:
\begin{equation}
    u_1(t) = \sin(2\pi \cdot 1.2t + 0.2),
\end{equation}
\begin{equation}
     u_2(t) = 0.8 \sin(2\pi \cdot 0.45t - 0.5).
\end{equation}
We define 10 measurement nodes ($y_1, \dots, y_{10}$) where each node is a non-linear combination of these latent drivers and a low-frequency temporal drift. Specifically, each measurement $y_i$ is generated as:
\begin{equation}
    y_i(t) = f_i(u_1(t), u_2(t)) + s_i(t) + \epsilon_i(t),
\end{equation}
where $f_i(\cdot)$ represents node-specific interactions, $s_i(t)$ represents a slow-moving seasonal component, and $\epsilon_i(t) \sim \mathcal{N}(0,\sigma_i^2)$ is additive white Gaussian noise. The noise variance $\sigma_i^2$ is dynamically adjusted for each node to maintain a Signal-to-Noise Ratio (SNR) of 45 dB. In total, we generate 3650 daily time steps for each measurement node.

\noindent \textbf{Synthetic District Heating (SDH):} We simulate a district heating network comprising 100 buildings (spatial nodes) with heterogeneous thermal characteristics using the \texttt{demandlib} library\footnote{\url{https://demandlib.readthedocs.io/en/latest/}}. The simulated network comprises 60\% single-family houses, 30\% multi-family houses, and 10\% commercial buildings to reflect realistic variability in building usage and consumption behavior. To introduce heterogeneity across the network, each building is assigned a unique seasonal demand profile with varying peak intensities and transition dynamics. In addition, we incorporate realistic consumption dynamics by including auto‑correlated fluctuations (modeled as a multiplicative first‑order autoregressive (AR(1)) noise with per‑building autocorrelation coefficients uniformly sampled from [0.75, 0.85] and relative magnitudes between 2.5\% and 5\%) as well as weekend‑specific behavioral variations. The simulation spans the period from 2010 to 2019 at daily resolution, resulting in 3652 observations per building. This dataset provides a controlled yet realistic environment for evaluating uncertainty estimation under heterogeneous spatial-temporal demand patterns.
 
\noindent \textbf{District Heating (EWZ):} This dataset consists of daily real-world district heating consumption measurements provided by Elektrizitätswerk der Stadt Zürich (EWZ), a Swiss energy utility company. The dataset contains heat demand observations from 48 buildings (spatial nodes) located in Switzerland from April 1, 2016, to October 31, 2025, resulting in 3501 daily observations per node. Measurements are in MWh and reflect the complexity of an operational district heating network, including heterogeneous consumption behavior, irregular occupancy patterns, seasonal demand fluctuations, and measurement noise. As a result, the dataset provides a realistic and challenging benchmark for evaluating uncertainty quantification methods in physical energy distribution systems.

\noindent \textbf{Electricity Consumption Load (ECL):} 
The ECL dataset\footnote{\url{https://archive.ics.uci.edu/dataset/321/electricityloaddiagrams20112014}} comprises electricity consumption measurements collected from clients in Portugal at 15-minute resolution. Although the original monitoring network is larger, several clients were excluded because their meters were not installed at the start of the observation period or contained significant gaps. After filtering for continuous measurements, 40 clients were retained and treated as spatial nodes. To ensure consistency with the other datasets and reduce high-frequency consumption noise, the original power measurements (kW) were converted to energy consumption (kWh) and aggregated to daily totals. We consider the period from January 1, 2011 to December 31, 2014, resulting in 1461 daily observations per client. The dataset captures heterogeneous electricity demand patterns across consumers and provides a realistic benchmark for evaluating spatial-temporal forecasting and uncertainty quantification methods under diverse consumption behaviors.

\noindent \textbf{Smart Meter Aggregates (CKW):} This dataset, provided by Centralschweizerische Kraftwerke AG (CKW)\footnote{\url{https://axsa4prod4publicdata4sa.blob.core.windows.net/$web/index.html\#dataset-a}}, consists of smart meter electricity consumption measurements aggregated at the postal code level to preserve consumer privacy. Out of the 116 municipalities originally available, we retain 93 that do not have significant missing values. We consider these 93 municipality-level spatial nodes over the period from December 31, 2020, to February 5, 2026.  Large industrial consumers are excluded in order  to focus on residential consumption dynamics. Original measurements are recorded at 15-minute resolution and aggregated to daily energy consumption values (kWh) to align with the temporal granularity of the other datasets. This results in 1863 daily observations for each spatial node. The dataset captures heterogeneous residential demand patterns across municipalities, providing a realistic large-scale benchmark for spatial-temporal uncertainty quantification.

\subsection{Baselines}

We compare STOIC against a diverse set of CP baselines spanning global, temporal, and spatially extended approaches. Split Conformal Prediction (SCP) \citep{papadopoulos2002inductive} provides finite-sample marginal coverage under exchangeability, but produces uniform prediction intervals by disregarding spatial and temporal heterogeneity. To account for temporal non-stationarity, Adaptive Conformal Inference (ACI) \citep{gibbs2021adaptive} and Sequential Predictive Conformal Inference (SPCI) \citep{xu2023sequential} dynamically adapt interval widths based on recent coverage errors and historical residuals, respectively. However, both methods model each time series independently and therefore fail to capture dependencies across nodes. To incorporate spatial structure, Localized Spatial Conformal Prediction (LSCP) \citep{jiang2024spatial} and Spatial-Temporal Conformal Prediction (STCP) \citep{jiang2025spatio} introduce spatially-aware calibration mechanisms. In contrast, STOIC employs a tabular FM to learn feature-conditional residual distributions from temporal and relational residual context, producing fully adaptive and context-aware PIs without task-specific retraining. To ensure a fair comparison, STOIC and all considered conformal baselines utilize the same point forecasting backbone described in Section \ref{sec:methodology}, and all evaluations are conducted at a target confidence level of $1-\alpha = 0.90$.

\subsection{Performance Metrics}

We evaluate the quality of the predicted intervals using standard metrics that assess both calibration and sharpness. Given a PI $C_{i,t}^\alpha = [\ell_{i,t}^\alpha, u_{i,t}^\alpha]$ at confidence level $1-\alpha$ and the corresponding ground-truth value $x_{i,t}$, we consider the following metrics.

\noindent \textbf{Prediction Interval Width (PI-Width):} PI-Width measures the sharpness of the interval and is defined as:
\begin{equation}
\text{(PI-Width)}_{i,t} = u_{i,t}^\alpha - \ell_{i,t}^\alpha.
\end{equation}
Lower values indicate tighter intervals.

\noindent \textbf{Coverage:} Coverage measures the proportion of observations contained within the predicted intervals:
\begin{equation}
\text{Coverage} = \mathbb{E}\left[ \mathbb{I}\left(x_{i,t} \in C_{i,t}^\alpha \right) \right],
\end{equation}
where $\mathbb{I}(\cdot)$ is the indicator function. Ideally, the empirical coverage should closely match the target confidence level.

\noindent \textbf{Winkler Score:} The Winkler score \citep{winkler1972decision} jointly evaluates PI-Width and  calibration by penalizing PIs that fail to cover the true value:
\begin{equation}
\text{Winkler Score}=
\begin{cases}
u_{i,t}^\alpha - \ell_{i,t}^\alpha, & \text{if } x_{i,t} \in C_{i,t}^\alpha, \\
(u_{i,t}^\alpha - \ell_{i,t}^\alpha) + \frac{2}{\alpha}(\ell_{i,t}^\alpha - x_{i,t}), & \text{if } x_{i,t} < \ell_{i,t}^\alpha, \\
(u_{i,t}^\alpha - \ell_{i,t}^\alpha) + \frac{2}{\alpha}(x_{i,t} - u_{i,t}^\alpha), & \text{if } x_{i,t} > u_{i,t}^\alpha.
\end{cases}
\end{equation}
The Winkler score rewards narrow prediction intervals while imposing increasingly large penalties for prediction failures. All metrics are averaged across all nodes and time steps to provide a global assessment of uncertainty quantification performance.

\section{Results and Discussion}
\label{sec:results}

\subsection{Quantitative Results}
This section presents the empirical evaluation of STOIC against several CP baselines across all datasets. All experiments were conducted at a target confidence level of $90\%$.

The results on the SCS dataset, detailed in Table \ref{tab:results_scs}, highlight substantial differences in calibration performance across methods. Notably, only STCP and STOIC achieve empirical coverage at or above the nominal 90\% target, whereas all remaining baselines suffer from systematic under-coverage, with coverage values around 88\%. Furthermore, methods such as SCP and ACI attempt to improve coverage by producing wider prediction intervals, yet still fail to reach the target coverage level, resulting in an unfavorable trade-off between reliability and sharpness. In contrast, STOIC achieves a PI-Width of 0.0309, which is approximately 20\% narrower than the strongest temporal baseline, while simultaneously achieving empirical coverage above the nominal target. This advantage is further reflected in the Winkler Score, where STOIC achieves the lowest value among all compared methods. These results demonstrate that STOIC is uniquely able to combine accurate calibration with highly informative prediction intervals, whereas competing approaches either exhibit under-coverage or require substantially wider intervals to achieve comparable performance.


\begin{table}[htbp]
    \centering
    \caption{Performance on Simulated Control Signals (SCS)}
    \label{tab:results_scs}
    \begin{tabular}{lccc}
        \toprule
        Method & Coverage (\%) & PI-Width ↓ & Winkler Score ↓ \\
        \midrule
        SCP & ${88.13} \pm 0.00$ \xmark & $0.0490 \pm 0.0000$ & $0.0826 \pm 0.0000$ \\
        SPCI & ${88.39} \pm 0.44$ \xmark & $0.0387 \pm 0.0004$ & $0.0588 \pm 0.0006$ \\
        ACI & ${88.83} \pm 0.00$ \xmark & $0.0517 \pm 0.0000$ & $0.0791 \pm 0.0000$ \\
        LSCP & ${88.01} \pm 0.19$ \xmark & $0.0444 \pm 0.0003$ & $0.0691 \pm 0.0005$ \\
        STCP & ${90.90} \pm 0.28$ \cmark & $0.0388 \pm 0.0003$ & $0.0549 \pm 0.0003$ \\
        \textbf{STOIC} & ${91.57} \pm 0.12$ \cmark & $\textbf{0.0309} \pm 0.0002$ & $\textbf{0.0435} \pm 0.0001$ \\
        \bottomrule
    \end{tabular}
\end{table}

The results on the SDH dataset, shown in Table \ref{tab:results_sdh}, indicate a different behavior compared to SCS. All methods achieve empirical coverage at or above the nominal 90\% level, indicating that calibration is less challenging on this dataset. As a result, the comparison is driven primarily by interval efficiency rather than coverage. Under these conditions, STOIC produces the narrowest prediction intervals among the compared methods while maintaining empirical coverage at or above the nominal level. This demonstrates that STOIC can translate equivalent calibration performance into substantially more precise uncertainty estimates, thereby providing a more favorable balance between reliability and informativeness.

\begin{table}[htbp]
    \centering
    \caption{Performance on Synthetic District Heating (SDH)}
    \label{tab:results_sdh}
    \begin{tabular}{lccc}
        \toprule
        Method & Coverage (\%) & PI-Width ↓ & Winkler Score ↓ \\
        \midrule
        SCP & ${90.78} \pm 0.00$ \cmark & $0.0082 \pm 0.0000$ & $0.0126 \pm 0.0000$ \\
        SPCI & ${90.04} \pm 0.27$ \cmark & $0.0077 \pm 0.0001$ & $0.0111 \pm 0.0000$ \\
        ACI & ${90.63} \pm 0.00$ \cmark & $0.0082 \pm 0.0000$ & $0.0126 \pm 0.0000$ \\
        LSCP & ${90.12} \pm 0.23$ \cmark & $0.0077 \pm 0.0001$ & $0.0114 \pm 0.0001$ \\
        STCP & ${90.59} \pm 0.12$ \cmark & $0.0075 \pm 0.0001$ & $0.0104 \pm 0.0000$ \\
        \textbf{STOIC} & ${91.05} \pm 0.04$ \cmark & $\textbf{0.0067} \pm 0.0000$ & $\textbf{0.0093} \pm 0.0000$ \\
        \bottomrule   
    \end{tabular}
\end{table}

For the real-world EWZ dataset, the results in Table \ref{tab:results_ewz} further demonstrate the effectiveness of STOIC. The method achieves an empirical coverage of 90.15\%,  closely matching the nominal target and resulting in the smallest coverage gap among the compared approaches. At the same time, STOIC produces substantially narrower prediction intervals than LSCP, despite the latter explicitly leveraging spatial neighborhoods (PI-Width: 0.3947 vs 0.4594). The resulting Winkler Score of 0.6478 is the lowest across all compared methods, demonstrating that STOIC consistently achieves a more favorable balance between calibration accuracy and interval sharpness on a real-world utility network dataset.

\begin{table}[htbp]
    \centering
    \caption{Performance on District Heating (EWZ)}
    \label{tab:results_ewz}
    \begin{tabular}{lccc}
        \toprule
        Method & Coverage (\%) & PI-Width ↓ & Winkler Score ↓ \\
        \midrule
        SCP  & ${92.56} \pm 0.00$ \cmark & $0.4391 \pm 0.0000$ & $0.6995 \pm 0.0000$ \\
        SPCI & ${90.88} \pm 0.24$ \cmark & $0.4154 \pm 0.0035$ & $0.6941 \pm 0.0041$ \\
        ACI  & ${92.12} \pm 0.00$ \cmark & $0.4266 \pm 0.0000$ & $0.6973 \pm 0.0000$ \\
        LSCP & ${90.90} \pm 0.13$ \cmark & $0.4594 \pm 0.0058$ & $0.7219 \pm 0.0005$ \\
        STCP & ${91.28} \pm 0.10$ \cmark & $0.4335 \pm 0.0045$ & $0.7028 \pm 0.0015$ \\
        \textbf{STOIC} & ${90.15} \pm 0.08$ \cmark & $\textbf{0.3947} \pm 0.0003$ & $\textbf{0.6478} \pm 0.0005$ \\
        \bottomrule
    \end{tabular}
\end{table}

The ECL evaluation, summarized in Table~\ref{tab:results_elc}, shows that STOIC and STCP achieve coverage close to the nominal 90\% level (90.52\% and 90.53\%, respectively), whereas SCP and ACI are more conservative and SPCI and LSCP exhibit under-coverage.
Although both STOIC and STCP achieve empirical coverage above the nominal target, STOIC achieves a substantially lower Winkler Score (1674.85 vs.\ 1813.12 for STCP), indicating more efficient uncertainty quantification.
While LSCP produces the narrowest prediction intervals, it fails to meet the target coverage level, demonstrating that increased sharpness comes at the expense of calibration.
Overall, STOIC achieves near-nominal coverage and the best overall interval quality, as reflected by its superior Winkler Score.

\begin{table}[htbp]
    \centering
    \caption{Performance on Electricity Consumption Load (ECL)}
    \label{tab:results_elc}
    \begin{tabular}{lccc}
        \toprule
        Method & Coverage (\%) & PI-Width ↓ & Winkler Score ↓ \\
        \midrule
        SCP   & ${93.65} \pm 0.00$ \cmark & $1288.53 \pm 0.00$ & $1874.88 \pm 0.00$ \\
        SPCI  & ${89.44} \pm 0.25$ \xmark & $1199.82 \pm 10.97$ & $1929.06 \pm 12.57$ \\
        ACI   & ${93.43} \pm 0.00$ \cmark & $1278.77 \pm 0.00$ & $1867.85 \pm 0.00$ \\
        LSCP  & ${89.27} \pm 0.09$ \xmark & $\textbf{1137.70} \pm 16.64$ & $1783.94 \pm 4.12$ \\
        STCP  & ${90.53} \pm 0.49$ \cmark & $1193.72 \pm 15.71$ & $1813.12 \pm 29.05$ \\
        \textbf{STOIC} & ${90.52} \pm 0.06$ \cmark & $1168.54 \pm 0.79$ & $\textbf{1674.85} \pm 2.09$ \\
        \bottomrule
    \end{tabular}
\end{table}

Finally, the results on the CKW dataset in Table \ref{tab:results_ckw} highlight a clear trade-off between coverage and interval efficiency. Only SCP, ACI, and STOIC achieve coverage at or above the nominal 90\% level, whereas SPCI, LSCP, and STCP exhibit under-coverage. Although STCP produces the narrowest prediction intervals, this gain is achieved at the expense of falling below the nominal coverage target. In contrast, STOIC maintains empirical coverage above the nominal target while producing intervals that remain competitive in width.  This advantage is reflected in the lowest Winkler Score among all compared methods, indicating that STOIC is able to preserve calibration without resorting to overly conservative uncertainty estimates. These results further demonstrate that STOIC delivers the highest-quality prediction intervals by combining near-nominal empirical coverage with efficient interval widths.
\begin{table}[htbp]
    \centering
    \caption{Performance on Smart Meter Aggregates (CKW)}
    \label{tab:results_ckw}
    \begin{tabular}{lccc}
        \toprule
        Method & Coverage (\%) & PI-Width ↓ & Winkler Score ↓ \\
        \midrule
        SCP & ${91.80} \pm 0.00$ \cmark & $3001.00 \pm 0.00$ & $3916.29 \pm 0.00$ \\
        SPCI & ${88.97} \pm 0.07$ \xmark & $2131.12 \pm 4.24$ & $3067.55 \pm 8.69$ \\
        ACI & ${91.25} \pm 0.00$ \cmark & $2937.59 \pm 0.00$ & $3911.71 \pm 0.00$ \\
        LSCP & ${89.06} \pm 0.09$ \xmark & $2156.72 \pm 7.21$ & $3132.60 \pm 2.64$ \\
        STCP & ${89.77} \pm 0.25$ \xmark & $\textbf{2123.03} \pm 9.65$ & $2936.83 \pm 4.06$ \\
        \textbf{STOIC} & ${90.89} \pm 0.03$ \cmark & $2156.48 \pm 0.37$ & $\textbf{2927.36} \pm 2.02$ \\
        \bottomrule
    \end{tabular}
\end{table}

\begin{figure*}[ht]
    \centering
    \includegraphics[width=\linewidth]{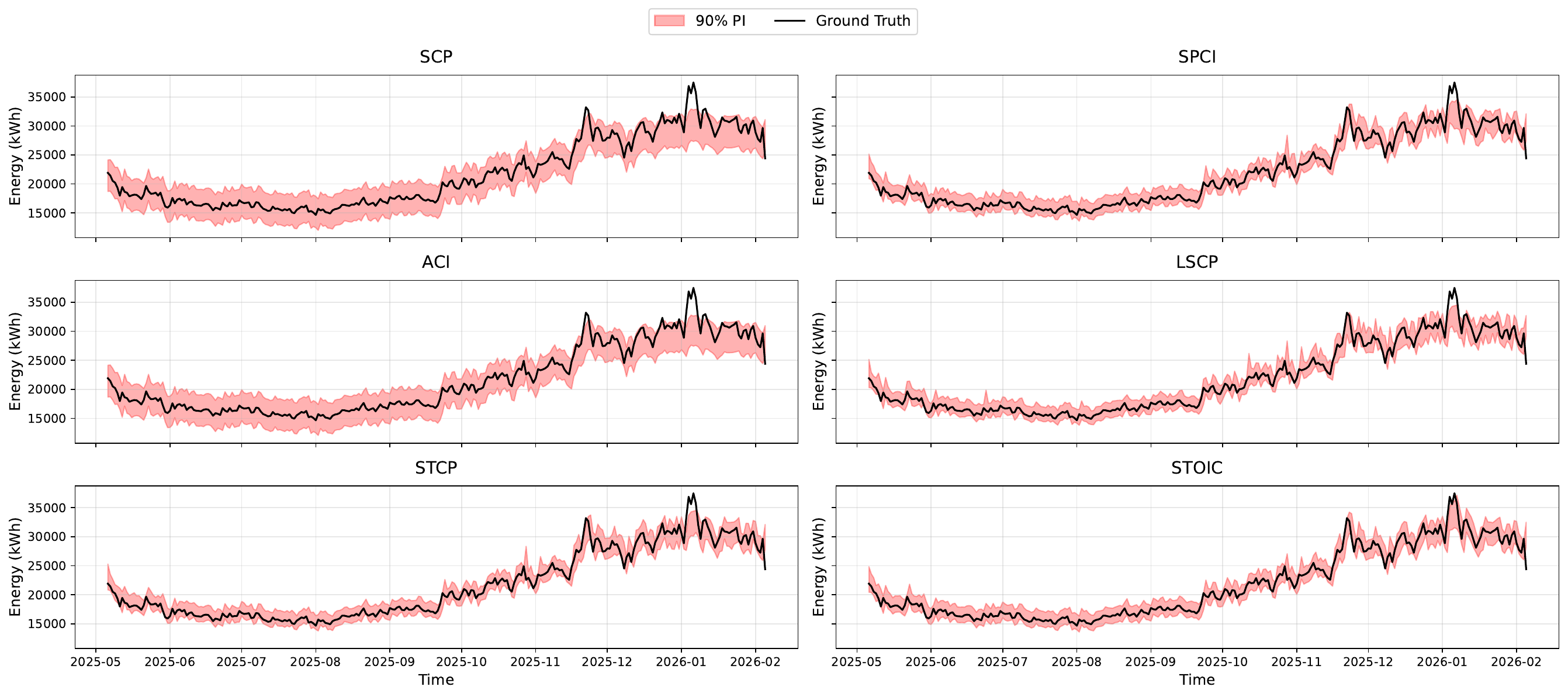}
    \caption{Visual comparison of 90\% PIs across STOIC and the considered conformal baselines for the CKW Smart Meter dataset. The shaded area represents the 90\% PI, and the black line indicates the ground truth energy consumption (kWh).}
    \label{fig:ckw_vis}
\end{figure*}

\begin{figure*}[ht]
    \centering
    \includegraphics[width=\linewidth]{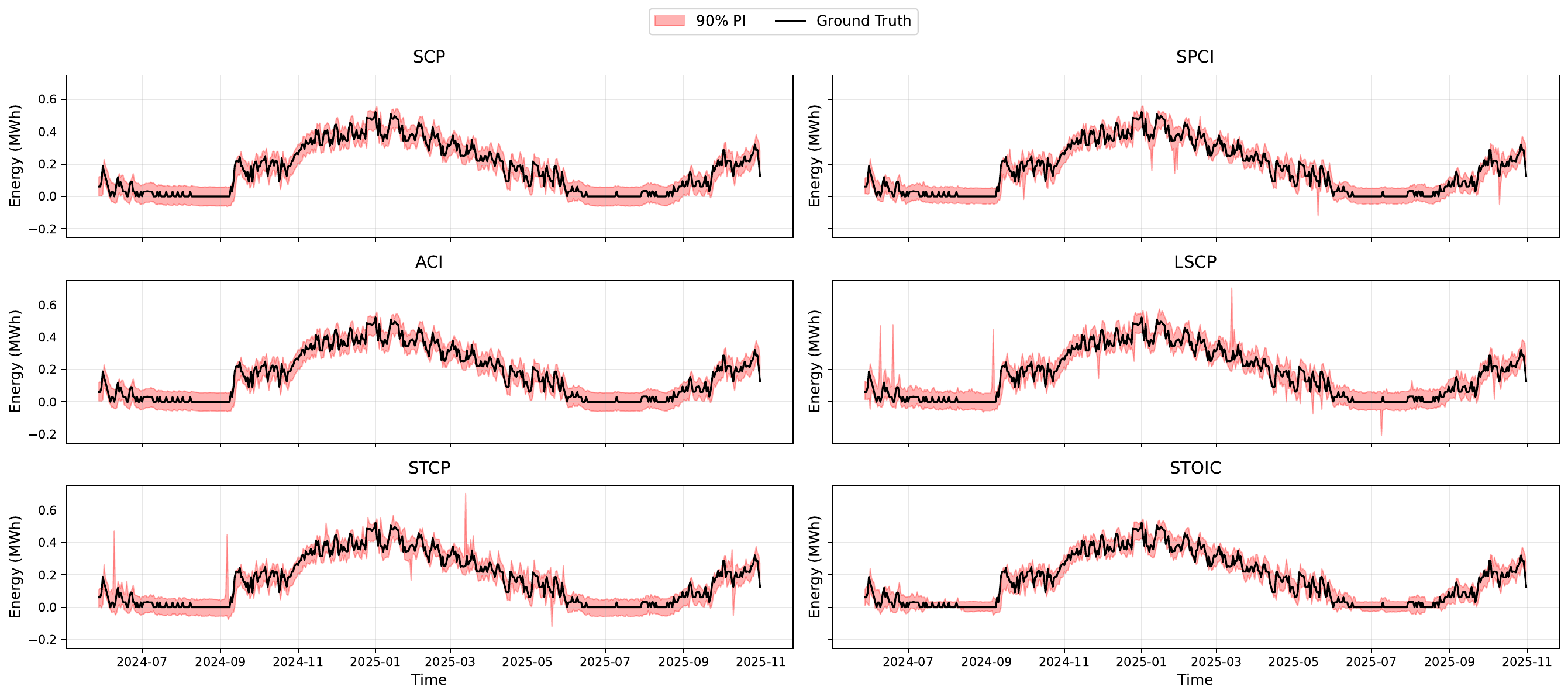}
    \caption{Visual comparison of 90\% PIs for the EWZ District Heating dataset. The plots illustrate the adaptive nature of the intervals (MWh) relative to seasonal consumption shifts and volatile demand spikes.}
    \label{fig:ewz_vis}
\end{figure*}

\subsection{Qualitative Results}

The qualitative analysis is conducted on the CKW and EWZ datasets because they represent two distinct forecasting settings. While CKW exhibits substantial temporal variability in electricity demand, EWZ is characterized by strong seasonality and noisier demand patterns. This allows the behavior of the prediction intervals to be examined under different sources of uncertainty.

Figure \ref{fig:ckw_vis} provides a qualitative comparison of the 90\% PIs for the CKW dataset, spanning from May 2025 to February 2026. While all methods successfully capture the broad upward trend in energy consumption during the winter months, there are distinct differences in interval efficiency. Standard methods like SCP and ACI exhibit relatively static or overly broad intervals that do not tightly track the daily volatility of the aggregated load. In contrast, STOIC produces substantially more adaptive intervals, remaining narrow during stable periods while expanding only when the underlying signal exhibits increased variability or rapid fluctuations. This behavior indicates that uncertainty is allocated where it is needed rather than uniformly across time. The resulting intervals more closely follow the local uncertainty structure of the data, providing sharper and more informative uncertainty estimates. This visual evidence supports the quantitative findings that STOIC achieves favorable interval efficiency together with near-nominal aggregate empirical coverage on the CKW dataset. The visualization for the EWZ dataset in Figure \ref{fig:ewz_vis} highlights how different methods respond to the extreme seasonality and noise typical of heat demand. A key observation in the LSCP and STCP plots is the presence of sharp, vertical spikes in the PIs during the early and mid-2025 periods. These spikes indicate that these methods are highly sensitive to localized spatial outliers, sometimes leading to excessively conservative (wide) intervals that may be less useful for practical demand forecasting and operational planning. STOIC, however, maintains a much smoother and more consistent adaptive envelope. By conditioning on a broader set of features through its tabular FM, STOIC avoids overreacting to individual noisy data points while maintaining a smooth and adaptive prediction interval throughout the transition from high-demand winter months to low-demand summer periods. Importantly, STOIC achieves this stability without sacrificing responsiveness to genuine changes in demand patterns, resulting in uncertainty estimates that are both more robust and more informative for operational decision-making.

\subsection{Coverage Assessment Across Nodes and Regimes}

Average empirical coverage can hide localized under-coverage across nodes and operating regimes. To evaluate whether the prediction intervals remain reliable beyond aggregate averages, we conduct an additional coverage assessment on the EWZ and CKW datasets. Specifically, we report node-wise coverage distributions, worst-node behavior, seasonal coverage, high-load coverage, and high-volatility coverage. These assessments are computed for all methods using the same test split and the same target coverage level of 90\%. Figure~\ref{fig:nodewise_coverage_diagnostics} reports the node-wise empirical coverage distributions. Each violin summarizes the distribution of coverage values across nodes for a given method, and the downward triangle marks the worst-covered node. On EWZ, STOIC achieves an average node-wise coverage of 90.15\% and a median node-wise coverage of 90.42\%, indicating that its typical node-level behavior is close to the nominal 90\% target. The worst-node coverage for STOIC is 85.44\%, showing that some localized under-coverage remains. Similar lower-tail behavior is also observed for other adaptive methods: SPCI, LSCP, and STCP have worst-node coverages of 85.06\%, 85.63\%, and 87.93\%, respectively. SCP and ACI are more conservative on average, with mean node-wise coverages of 92.56\% and 92.12\%, but this comes with higher aggregate coverage rather than tighter target calibration.

\begin{figure}[ht]
    \centering
    \begin{subfigure}{0.7\textwidth}
        \centering
        \includegraphics[width=\linewidth]{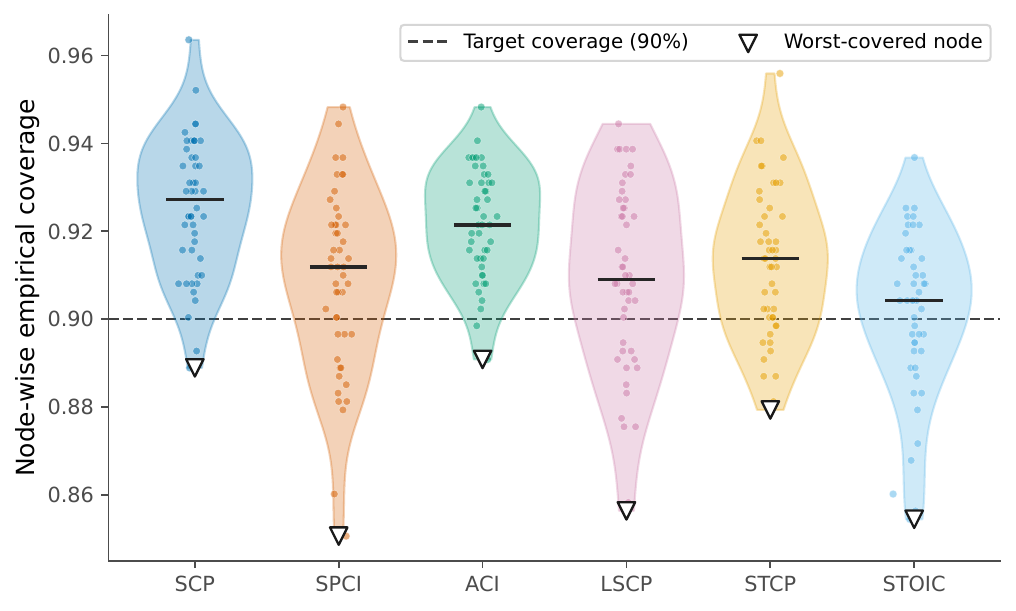}
        \caption{EWZ}
        \label{fig:nodewise_coverage_ewz}
    \end{subfigure}
    \hfill
    \begin{subfigure}{0.7\textwidth}
        \centering
        \includegraphics[width=\linewidth]{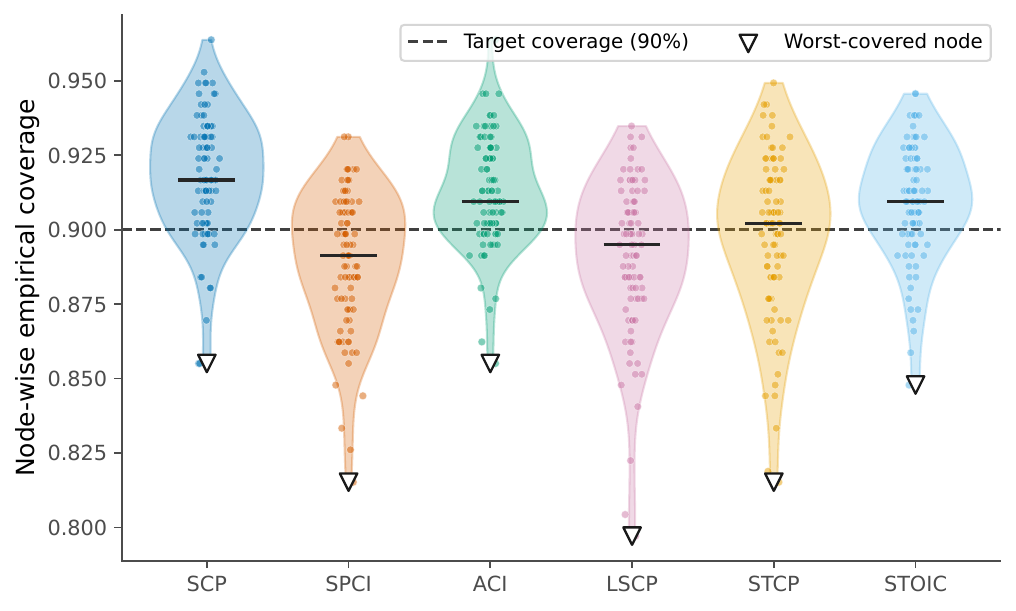}
        \caption{CKW}
        \label{fig:nodewise_coverage_ckw}
    \end{subfigure}
    \caption{Node-wise coverage assessment on EWZ and CKW. Each violin shows the distribution of empirical coverage across nodes for one method. The horizontal dashed line denotes the nominal 90\% coverage level, and the downward triangle marks the worst-covered node.}
    \label{fig:nodewise_coverage_diagnostics}
\end{figure}

On CKW, STOIC achieves an average node-wise coverage of 90.89\% and a median node-wise coverage of 90.94\%, again remaining close to the target coverage level at the node level. Its worst-node coverage is 84.78\%, while the 10th percentile of node-wise coverage is 88.41\%. Compared with SPCI, LSCP, and STCP, whose worst-node coverages are 81.52\%, 79.71\%, and 81.52\%, respectively, STOIC exhibits a less severe lower tail. SCP and ACI provide slightly higher worst-node coverage, both at 85.51\%, but they also tend to be more conservative on average. Overall, the node-wise assessments show that STOIC maintains coverage close to the nominal level for the typical node, while also revealing that worst-node under-coverage remains an important challenge for all methods. Figure~\ref{fig:regime_coverage_diagnostics} reports coverage across operational regimes. We evaluate winter and summer coverage to measure seasonal reliability, high-load coverage on samples whose demand lies in the top 10\% of the test distribution, and high-volatility coverage on samples whose absolute one-step change lies in the top 10\% of the test distribution. These regimes capture different failure modes: seasonal coverage measures calendar-dependent shifts, high-load coverage measures performance under large demand values, and high-volatility coverage measures performance during rapidly changing conditions.

\begin{figure}[ht]
    \centering
    \begin{subfigure}{0.75\textwidth}
        \centering
        \includegraphics[width=\linewidth]{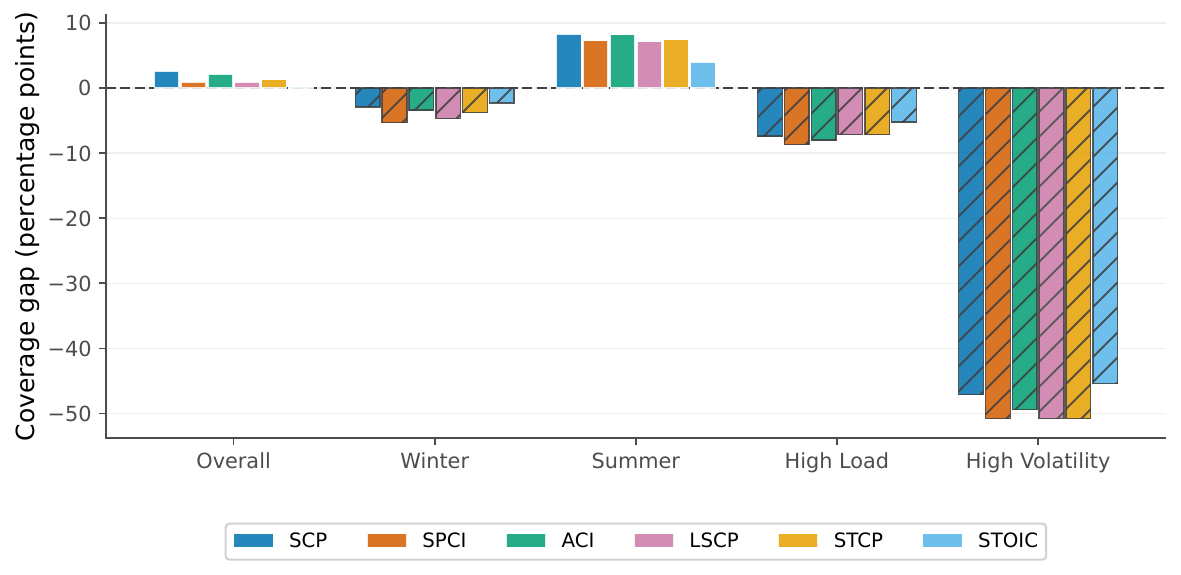}
        \caption{EWZ}
        \label{fig:regime_coverage_ewz}
    \end{subfigure}
    \hfill
    \begin{subfigure}{0.75\textwidth}
        \centering
        \includegraphics[width=\linewidth]{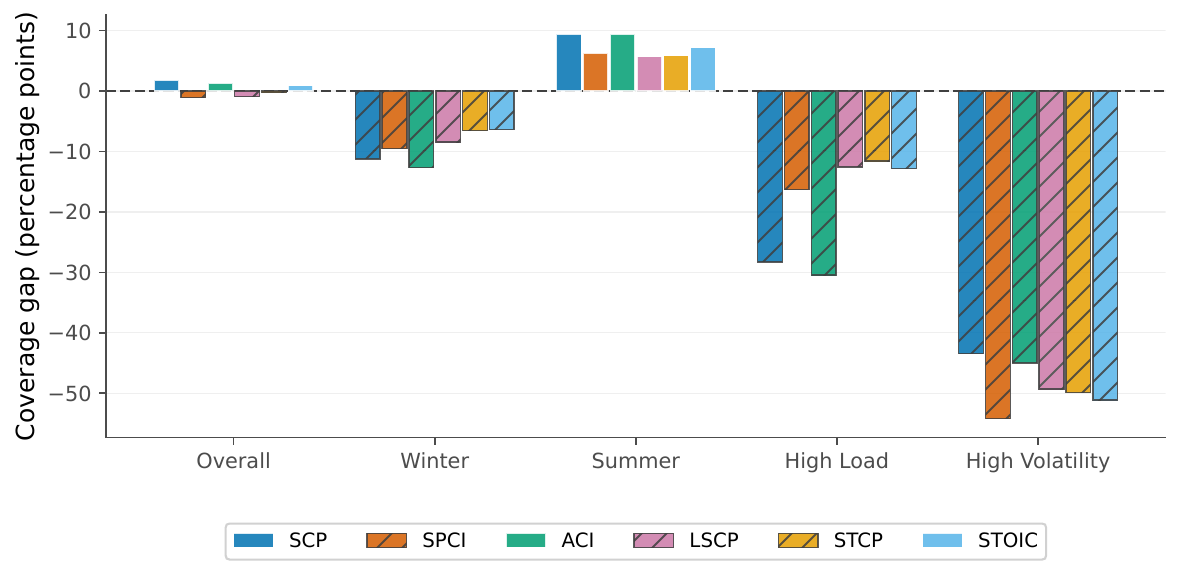}
        \caption{CKW}
        \label{fig:regime_coverage_ckw}
    \end{subfigure}
    \caption{Regime-wise coverage assessment on EWZ and CKW. Coverage is reported for winter, summer, high-load, and high-volatility regimes. The dashed reference line denotes the nominal 90\% coverage level.}
    \label{fig:regime_coverage_diagnostics}
\end{figure}

The seasonal assessments reveal a consistent asymmetry between winter and summer. On EWZ, all methods achieve higher coverage in summer than in winter. STOIC obtains the highest winter coverage among the evaluated methods, with 87.71\%, compared with 87.11\% for SCP, 86.60\% for ACI, 86.27\% for STCP, 85.30\% for LSCP, and 84.72\% for SPCI. In summer, all methods exceed the nominal coverage level; STOIC obtains 93.95\%, while the other methods range from 97.14\% to 98.25\%. This indicates that summer intervals tend to be conservative, whereas winter remains the more difficult season for calibration. A similar pattern appears on CKW. STOIC again achieves the highest winter coverage, with 83.65\%, followed closely by STCP with 83.50\%. The remaining methods have lower winter coverage, ranging from 77.39\% for ACI to 81.56\% for LSCP. In contrast, summer coverage is substantially above the nominal level for all methods, ranging from 95.70\% to 99.40\%. This seasonal gap suggests that the main calibration difficulty is concentrated in winter, likely due to stronger demand variability and more challenging operating conditions. The high-load and high-volatility assessments show that extreme regimes remain challenging. On EWZ, STOIC achieves the best high-load coverage, with 84.77\%, compared with 82.86\% for LSCP, 82.82\% for STCP, 82.62\% for SCP, 82.00\% for ACI, and 81.34\% for SPCI. On CKW, high-load coverage is highest for STCP at 78.42\%, followed by LSCP at 77.46\% and STOIC at 77.15\%. These results show that STOIC is competitive under high-demand conditions, but all methods still fall below the nominal 90\% target in the high-load regime. High-volatility coverage is substantially lower for all methods on both datasets. On EWZ, STOIC achieves the highest high-volatility coverage, with 44.57\%, but this is still far below the nominal target. On CKW, high-volatility coverage ranges from 35.83\% to 46.58\%, with SCP and ACI obtaining the highest values and STOIC reaching 38.86\%. These results indicate that abrupt temporal changes remain the most difficult regime for uncertainty calibration. Overall, the regime-wise assessments show that STOIC improves or remains competitive in several important regimes, especially winter and high-load cases, but they also reveal that high-volatility periods remain an open challenge for all methods. These assessments provide a more detailed view of calibration behavior than average coverage alone. STOIC achieves near-nominal average node-wise coverage on both datasets and is competitive in the lower tail of the node-wise coverage distribution. The regime-wise analysis further shows that STOIC performs particularly well in winter and high-load regimes, but also that high-volatility periods remain difficult for all methods.

\subsection{Calibration Across Nominal Coverage Levels}
\label{sec:calibration_curves}

To assess whether the methods remain calibrated across different confidence levels, we additionally compute empirical coverage at nominal levels of $70\%$, $75\%$, $80\%$, $85\%$, $90\%$, and $95\%$ on the EWZ and CKW datasets. For every nominal level, the prediction intervals of each method are recomputed using the corresponding miscoverage level $\alpha$. Figure~\ref{fig:calibration_curves} compares nominal and empirical coverage, with the diagonal indicating ideal calibration. On EWZ, STOIC follows the ideal calibration line closely across the entire evaluated range. In contrast, most competing methods are systematically conservative at the lower and intermediate nominal levels, producing empirical coverage above the requested target. On CKW, the calibration behavior is more challenging and varies across confidence levels. STOIC remains close to the nominal target at the higher confidence levels, while at lower nominal levels, it is moderately conservative. SCP and ACI exhibit more substantial over-coverage across most levels, while SPCI and LSCP tend to become slightly under-covered at the highest confidence levels. Overall, the results show that calibration behavior depends on both the dataset and the requested confidence level. STOIC is particularly well-calibrated on EWZ and maintains near-nominal high-confidence coverage on CKW without exhibiting the strong conservativeness of SCP and ACI.

\begin{figure}[t]
    \centering
    \includegraphics[width=\linewidth]{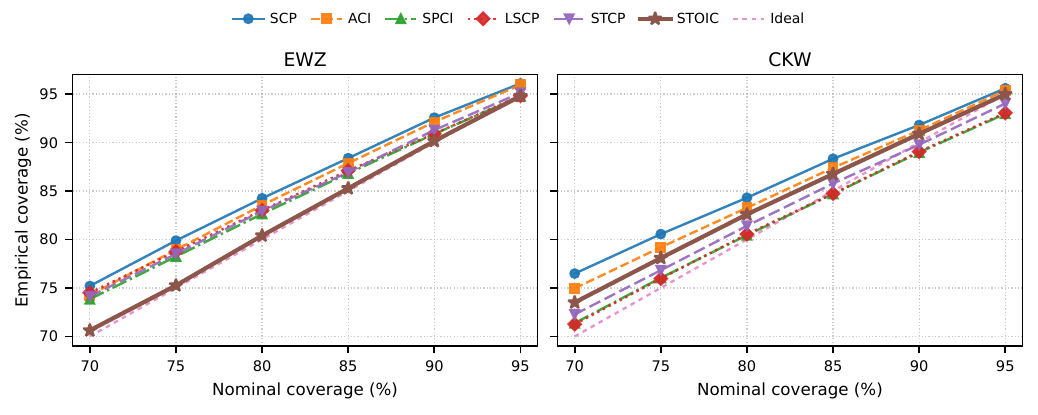}
    \caption{Calibration curves on the EWZ and CKW datasets across nominal
    coverage levels from $70\%$ to $95\%$. The diagonal denotes ideal
    calibration, where empirical coverage equals nominal coverage.
    Empirical coverage is averaged across all nodes and test time steps.}
    \label{fig:calibration_curves}
\end{figure}

\subsection{Comparison with a Time-Series Foundation Model and Direct Quantile Regression}

Recent advances in time-series foundation models have established large-scale pretrained forecasting models as strong benchmarks for both point forecasting and uncertainty quantification. To assess the benefits of STOIC beyond large-scale pretraining, we compare it against TimesFM \citep{das2024decoder}, a general-purpose pretrained time-series foundation model with probabilistic forecasting capabilities. In addition, we introduce a direct quantile-regression baseline, denoted \textbf{QSTGNN}, which uses the same spatial-temporal graph forecasting architecture as STOIC but does not use STOIC's post-hoc residual calibration stage. Instead, QSTGNN directly predicts the lower and upper conditional quantiles using the pinball loss. This additional baseline allows us to distinguish the benefit of the proposed spatial-temporal architecture from the benefit provided specifically by STOIC's post-hoc residual calibration.

TimesFM is a 200M-parameter decoder-only transformer pretrained on a large collection of real-world and synthetic time series. Its probabilistic forecasts provide multiple quantile levels, including the 10th and 90th percentiles. For QSTGNN, the forecasting head is similarly trained to estimate the 0.1 and 0.9 conditional quantiles directly using pinball loss. Consequently, all three methods can be evaluated at the same target coverage of 80\%, corresponding to the interval between the 0.1 and 0.9 quantiles.

We conduct this comparison on the two real-world datasets, EWZ and CKW, which represent distinct energy-demand forecasting settings. TimesFM is evaluated in a zero-shot setting, whereas QSTGNN and STOIC exploit the spatial-temporal structure of the underlying energy network. Importantly, QSTGNN does not employ any post-hoc conformal correction: its prediction intervals are obtained directly from the learned quantiles. In contrast, STOIC applies its residual calibration procedure at the 80\% target level while retaining the same forecasting backbone. The comparison therefore evaluates three complementary approaches to uncertainty quantification: large-scale pretrained probabilistic forecasting, task-specific direct quantile regression, and post-hoc residual calibration. The results are summarized in Table~\ref{tab:timesfm}.

\begin{table}[htbp]
    \centering
    \caption{Comparison with TimesFM and direct quantile regression (QSTGNN) on the EWZ and CKW datasets at 80\% target coverage.}
    \label{tab:timesfm}
    \begin{tabular}{llccc}
        \toprule
        Dataset & Method & Coverage (\%) & PI-Width $\downarrow$ & Winkler Score $\downarrow$ \\
        \midrule
        
        \multirow{3}{*}{EWZ}
        & TimesFM & {73.47} \xmark & 0.2748 & 0.7798 \\
        & QSTGNN & {78.27} \xmark & \textbf{0.2599} & 0.5005 \\
        & \textbf{STOIC} & {80.40} \cmark & 0.2954 & \textbf{0.4830} \\
        
        \midrule
        
        \multirow{3}{*}{CKW}
        & TimesFM & {71.68} \xmark & \textbf{1535.38} & 3569.99 \\
        & QSTGNN & {79.56} \xmark & 1827.55 & 2519.81 \\
        & \textbf{STOIC} & {82.59} \cmark & 1676.80 & \textbf{2382.27} \\
        
        \bottomrule
    \end{tabular}
\end{table}

The results reveal a clear distinction between interval sharpness and reliable coverage. TimesFM generates relatively narrow prediction intervals, particularly on CKW, but substantially undercovers the observations on both datasets. At the nominal 80\% coverage level, TimesFM achieves only 73.47\% coverage on EWZ and 71.68\% on CKW. This undercoverage is reflected in substantially larger Winkler scores of 0.7798 and 3569.99, respectively, indicating that the narrower intervals frequently fail to contain the realized observations.

QSTGNN substantially improves uncertainty estimation compared with TimesFM. On EWZ, QSTGNN achieves 78.27\% coverage while producing the narrowest intervals among all evaluated methods, with an average width of 0.2599. Its Winkler score is reduced from 0.7798 for TimesFM to 0.5005. Similarly, on CKW, QSTGNN reaches 79.56\% coverage, only 0.44 percentage points below the nominal target, and reduces the Winkler score from 3569.99 to 2519.81. These results demonstrate that explicitly modeling the spatial-temporal dependencies of the energy network and directly optimizing the relevant conditional quantiles already provides a substantial advantage over a general-purpose pretrained forecasting model.

Nevertheless, direct quantile regression does not achieve the target marginal coverage. This is particularly visible on EWZ, where QSTGNN remains 1.73 percentage points below the desired 80\% coverage. STOIC addresses this limitation through its residual calibration step. It achieves 80.40\% coverage on EWZ and 82.59\% on CKW while obtaining the lowest Winkler scores on both datasets, namely 0.4830 and 2382.27, respectively. On CKW, STOIC additionally produces narrower intervals than QSTGNN (1676.80 versus 1827.55) despite achieving higher coverage. Thus, the improvement provided by STOIC cannot be attributed solely to increasing interval width.

The comparison between QSTGNN and STOIC is particularly informative because both approaches share the same underlying spatial-temporal forecasting architecture. Their difference lies primarily in how uncertainty is constructed: QSTGNN learns the 0.1 and 0.9 quantiles directly through pinball loss, whereas STOIC augments the forecasting pipeline with TabPFN. The considerably smaller performance gap between QSTGNN and STOIC compared with the gap between TimesFM and STOIC indicates that the graph-aware spatial-temporal representation itself contributes substantially to uncertainty quality. At the same time, the remaining improvement in coverage and Winkler score obtained by STOIC highlights the additional value of explicit calibration.

Overall, these results suggest that large-scale pretraining alone does not ensure reliable uncertainty estimates for structured energy-demand forecasting. Direct quantile regression with a task-specific spatial-temporal graph model provides a strong baseline and achieves near-nominal coverage, particularly on CKW. However, QSTGNN exhibits mild undercoverage in the evaluated datasets, and direct quantile regression does not provide a finite-sample distribution-free coverage guarantee. STOIC combines the advantages of domain-specific spatial-temporal modeling with explicit residual calibration, resulting in empirical coverage close to or above the nominal target along with the lowest Winkler scores on both real-world datasets.

\subsection{Ablation Study}
\subsubsection{Effects of Features and Calibration Model}

To better understand which components drive the performance of STOIC, we conduct an ablation study on the EWZ and CKW datasets. The study evaluates both the contribution of the proposed feature representations and the impact of the calibration model used to construct prediction intervals. First, we independently remove graph-based and temporal features to assess their respective contributions to uncertainty-quantification performance. Second, we isolate the role of the tabular foundation model by replacing TabPFN with alternative quantile estimators trained on exactly the same residual feature representation.

Specifically, we compare TabPFN against four same-feature calibration baselines: Quantile Random Forests (QRF), Gradient-Boosted Quantile Regression (GBQR), a Quantile Multilayer Perceptron (QMLP), and Quantile Linear Regression (QLR). QRF serves as a nonparametric ensemble baseline, while GBQR provides a nonlinear boosted-tree quantile estimator. QMLP is implemented as a neural quantile-regression model with two hidden layers of 128 and 64 units and is trained using the pinball loss to predict the lower and upper conditional quantiles. QLR evaluates whether the engineered residual features are sufficient for a simple linear conditional quantile model. Because all four baselines use the same temporal and graph-based residual features as STOIC, this comparison isolates the benefit of the foundation-model-based calibration strategy from the benefit provided by the feature representation itself.

All variants are evaluated using empirical coverage, PI-Width, and Winkler score. The ablation study therefore provides insight into the relative importance of graph information, temporal context, and foundation-model-based calibration. By systematically removing individual feature components and replacing the calibration model while keeping the remainder of the pipeline unchanged, we can identify the contribution of each component. The results, compared with the full STOIC model, are reported in Table~\ref{tab:ablation_results}.

\begin{table}[htbp]
    \centering
    \caption{Ablation study results on the EWZ and CKW datasets.}
    \label{tab:ablation_results}
    \begin{tabular}{llccc}
        \toprule
        Dataset & Configuration & Coverage (\%) & PI-Width $\downarrow$ & Winkler Score $\downarrow$ \\
        \midrule

        \multirow{7}{*}{EWZ}
        & Full STOIC Model
        & $90.15 \pm 0.08$ \cmark
        & $0.3947 \pm 0.0003$
        & $\mathbf{0.6478} \pm 0.0005$ \\

        & w/o Graph Features
        & $89.96 \pm 0.02$ \xmark
        & $0.3953 \pm 0.0003$
        & $0.6563 \pm 0.0005$ \\

        & w/o Temporal Features
        & $90.21 \pm 0.01$ \cmark
        & $0.3980 \pm 0.0004$
        & $0.6727 \pm 0.0002$ \\

        & w/o TabPFN (QRF)
        & $90.21 \pm 0.06$ \cmark
        & $0.4503 \pm 0.0027$
        & $0.7186 \pm 0.0013$ \\

        & w/o TabPFN (GBQR)
        & $84.78 \pm 0.02$ \xmark
        & $0.3759 \pm 0.0049$
        & $0.7023 \pm 0.0042$ \\

        & w/o TabPFN (QMLP)
        & $86.97 \pm 0.13$ \xmark
        & $0.4121 \pm 0.0037$
        & $0.7017 \pm 0.0023$ \\

        & w/o TabPFN (QLR)
        & $83.06 \pm 0.00$ \xmark
        & $0.4105 \pm 0.0000$
        & $0.9075 \pm 0.0000$ \\

        \midrule

        \multirow{7}{*}{CKW}
        & Full STOIC Model
        & $90.89 \pm 0.03$ \cmark
        & $2156.48 \pm 0.37$
        & $\mathbf{2927.36} \pm 2.02$ \\

        & w/o Graph Features
        & $90.72 \pm 0.04$ \cmark
        & $2194.87 \pm 0.45$
        & $3002.30 \pm 0.63$ \\

        & w/o Temporal Features
        & $90.73 \pm 0.02$ \cmark
        & $2186.19 \pm 0.70$
        & $2961.14 \pm 3.67$ \\

        & w/o TabPFN (QRF)
        & $89.09 \pm 0.04$ \xmark
        & $2139.63 \pm 3.75$
        & $2981.68 \pm 2.63$ \\

        & w/o TabPFN (GBQR)
        & $84.01 \pm 0.04$ \xmark
        & $1912.11 \pm 1.02$
        & $3233.70 \pm 2.95$ \\

        & w/o TabPFN (QMLP)
        & $87.74 \pm 0.31$ \xmark
        & $2221.12 \pm 24.45$
        & $3195.21 \pm 18.00$ \\

        & w/o TabPFN (QLR)
        & $75.30 \pm 0.00$ \xmark
        & $1806.93 \pm 0.00$
        & $4148.69 \pm 0.00$ \\

        \bottomrule
    \end{tabular}
\end{table}

On the EWZ dataset, removing graph-based features slightly reduces coverage to 89.96\% and increases the Winkler score from 0.6478 to 0.6563. Removing temporal features leads to a larger increase in the Winkler score, reaching 0.6727, which suggests that temporal residual dynamics are particularly important for building-level uncertainty estimation. Graph features still provide complementary gains, but temporal residual features appear to be the dominant source of information on EWZ.

Replacing TabPFN with alternative same-feature quantile estimators leads to a more pronounced degradation. QRF preserves the nominal coverage level, achieving 90.21\%, but produces considerably wider intervals and a higher Winkler score than the full STOIC model. This indicates that the engineered residual features are informative, but QRF is less efficient than TabPFN in translating these features into sharp prediction intervals with favorable empirical coverage.

GBQR produces the narrowest intervals among the calibration-model replacements on EWZ, but its coverage drops substantially to 84.78\%, indicating severe undercoverage. QMLP improves coverage relative to GBQR and QLR, reaching 86.97\%, and obtains a Winkler score of 0.7017. This score is slightly lower than that of GBQR and lower than that of QRF, but QMLP still falls 3.03 percentage points below the nominal coverage level. QLR performs worst on EWZ, with coverage of 83.06\% and a substantially higher Winkler score of 0.9075. These results show that increasing the flexibility of a task-specific quantile estimator does not, by itself, ensure empirical coverage close to the nominal target. The gains of STOIC are therefore not solely attributable to its feature representation; the calibration model also plays a central role.

On the larger-scale CKW dataset, removing graph-based or temporal features again degrades interval quality while maintaining coverage close to the nominal level. Removing graph-based features yields a Winkler score of 3002.30, compared with 2961.14 when temporal features are removed, suggesting that relational information is especially important for the municipality-level electricity network. This is consistent with the intuition that spatial dependencies become more pronounced at larger aggregation scales.

The calibration-model ablations on CKW further highlight the importance of TabPFN. QRF is the strongest task-specific replacement on this dataset, achieving coverage of 89.09\% and a Winkler score of 2981.68, although it still underperforms the full STOIC model. GBQR obtains a considerably smaller average width of 1912.11, but its coverage drops to 84.01\%, resulting in a higher Winkler score of 3233.70.

QMLP achieves coverage of 87.74\%, which is higher than that of GBQR and QLR but remains 2.26 percentage points below the nominal level. It also produces the widest intervals among the calibration-model replacements, with an average PI-Width of 2221.12, while its Winkler score reaches 3195.21. Therefore, although the neural quantile model improves coverage relative to the boosted and linear alternatives, its increased interval width does not compensate sufficiently for the remaining undercoverage. QLR undercovers most severely, with coverage of only 75.30\% and a Winkler score of 4148.69. Although GBQR and QLR produce narrower intervals, their reduced widths come at the cost of substantial undercoverage.

Across both datasets, the results indicate that STOIC's performance arises from the combination of informative residual features and foundation-model-based calibration. The feature ablations confirm that both temporal and graph-based residual context contribute to uncertainty estimation. The calibration-model ablations further show that replacing TabPFN with QRF, GBQR, QMLP, or QLR generally reduces calibration quality, interval efficiency, or both.

QRF provides coverage closest to the nominal target among the task-specific calibration replacements and is the strongest overall replacement on CKW. QMLP demonstrates that a nonlinear neural quantile estimator can improve upon GBQR and QLR in terms of coverage, but it still fails to reach the nominal level on either dataset. TabPFN consequently achieves the best overall trade-off between coverage, interval width, and Winkler score. These findings support the role of TabPFN as a key component of STOIC rather than a simple implementation choice, particularly in the low-to-moderate data regime considered here, where its pretrained inductive bias can compensate for limited node-specific calibration data. With substantially larger calibration sets, however, task-specific quantile regressors such as QMLP may become more competitive.

\subsubsection{Sensitivity to Calibration Set Size}
\label{sec:calibration_sensitivity}

We examine how the size of the calibration set affects STOIC’s performance on the SCS dataset. We subsample this set to create calibration subsets ranging from 10\% to the full size, while keeping the test set unchanged.  For each subset, TabPFN is recalibrated and the resulting prediction intervals are evaluated using coverage, PI‑Width, and the Winkler score at the 90\% target level. Figure~\ref{fig:cal_sens} shows the three metrics as functions of the calibration fraction.  Even when only 10\% of the available calibration data is used, STOIC achieves a coverage of 88.92\%, remaining remarkably close to the nominal 90\% target. This result highlights the strength of the proposed calibration framework and demonstrates that uncertainty estimates with empirical coverage close to the nominal target can be obtained from very limited calibration data.  As the calibration set grows, coverage stabilizes around 90\% and eventually slightly exceeds it, while the average interval width steadily decreases and the Winkler score improves.  Importantly, the gains in sharpness are achieved without compromising calibration, indicating that additional calibration samples are primarily used to reduce uncertainty rather than to correct systematic coverage errors. This trend is consistent with more precise conditional quantile estimation as additional calibration samples are provided, yielding sharper intervals without sacrificing empirical coverage.

The observed behavior is particularly attractive from a practical deployment perspective. Many real-world forecasting systems initially operate with only a small amount of calibration data, especially when new assets, regions, or forecasting tasks are introduced. The results demonstrate that STOIC remains effective even in such low-data regimes and improves gracefully as more calibration data become available. Furthermore, increasing the calibration set size provides a predictable improvement in sharpness, making STOIC well‑suited for applications where historical data accumulates over time. Overall, these results demonstrate that STOIC combines strong calibration efficiency with robust uncertainty quantification, maintaining near-nominal coverage even under severe calibration data constraints while consistently benefiting from additional data. These results suggest potential utility in settings where only limited calibration data are initially available.

\begin{figure}[h]
\centering
\includegraphics[width=\linewidth]{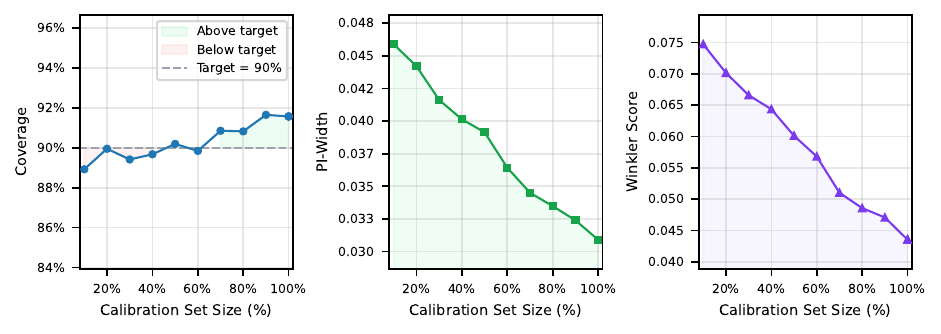}
\caption{Sensitivity of STOIC to calibration set size on the SCS dataset.}
\label{fig:cal_sens}
\end{figure}

\subsubsection{Effect of Graph Construction}

To further evaluate the role of relational structure in STOIC, we conduct an additional graph ablation study on the EWZ and CKW datasets. In the main STOIC model, graph-based residual features are computed using the learned graph. To assess whether this learned structure is beneficial, we replace it with four alternative graph constructions while keeping the rest of the pipeline unchanged. Specifically, we consider: (i) a fully connected graph, (ii) a correlation-based graph constructed from the training time series with the same sparsity level as the learned graph, (iii) a correlation-based graph constructed from calibration residuals with the same sparsity level, and (iv) a random binary graph with the same sparsity level. This design isolates the effect of graph construction while controlling for graph density.

\begin{table}[h]
\centering
\caption{Effect of graph construction on STOIC performance on the EWZ and CKW datasets. All graph variants use the same point forecaster, residual features, and TabPFN calibration pipeline. Correlation-based and random graphs are matched to the learned graph in sparsity.}
\label{tab:graph_ablation}
\begin{tabular}{llccc}
\toprule
Dataset & Graph Construction & Coverage (\%) & PI-Width $\downarrow$ & Winkler Score $\downarrow$ \\
\midrule
\multirow{5}{*}{EWZ}
& Learned graph & ${90.15} \pm 0.08$ \cmark & $0.3947 \pm 0.0003$ & $\textbf{0.6478} \pm 0.0005$ \\
& Fully connected graph
& $89.48 \pm 0.06$ \xmark
& $0.3914 \pm 0.0006$
& $0.6521 \pm 0.0004$ \\

& Training-correlation graph
& $89.98 \pm 0.07$ \xmark
& $0.3994 \pm 0.0006$
& $0.6540 \pm 0.0002$ \\

& Calibration-residual correlation graph
& $90.25 \pm 0.08$ \cmark
& $0.3950 \pm 0.0001$
& $0.6489 \pm 0.0015$ \\

& Random graph
& $89.80 \pm 0.11$ \xmark
& $0.3922 \pm 0.0031$
& $0.6499 \pm 0.0014$ \\
\midrule
\multirow{5}{*}{CKW}
& Learned graph & ${90.89} \pm 0.03$ \cmark & $2156.48 \pm 0.37$ & $\textbf{2927.36} \pm 2.02$ \\
& Fully connected graph
& $90.26 \pm 0.04$ \cmark
& $2179.36 \pm 0.62$
& $2976.08 \pm 2.09$ \\

& Training-correlation graph
& $90.79 \pm 0.04$ \cmark
& $2161.27 \pm 0.09$
& $2949.24 \pm 2.03$ \\

& Calibration-residual correlation graph
& $90.72 \pm 0.04$ \cmark
& $2164.68 \pm 0.37$
& $2967.96 \pm 0.47$ \\

& Random graph
& $90.83 \pm 0.02$ \cmark
& $2172.86 \pm 0.88$
& $2969.70 \pm 2.16$ \\
\bottomrule
\end{tabular}
\end{table}

The learned graph achieves the best overall Winkler score on both datasets while maintaining empirical coverage close to the nominal 90\% level. On EWZ, the learned graph obtains a coverage of 90.15\%, an average width of 0.3947, and the lowest Winkler score of 0.6478. The calibration-residual correlation graph performs very similarly, with slightly higher coverage and a marginally higher Winkler score. This is expected because residual correlations are directly aligned with the uncertainty-calibration task and therefore provide informative relational structure for residual aggregation.

The training-correlation graph also achieves coverage close to the target level on EWZ, but it produces wider intervals and a higher Winkler score than the learned graph. This suggests that correlations in the raw training time series are informative, but they do not fully capture the residual dependence structure relevant for uncertainty calibration. The fully connected graph produces the narrowest intervals on EWZ, but it slightly undercovers the nominal level and has a worse Winkler score than the learned graph. This indicates that aggregating over all nodes may dilute useful local relational information and reduce the specificity of graph-based residual features.

On CKW, the same general pattern holds. The learned graph achieves the best Winkler score, 2927.36, while maintaining coverage of 90.89\%. The training-correlation graph is the closest alternative, with coverage of 90.79\%, width of 2161.27, and Winkler score of 2949.24. The calibration-residual correlation graph and random graph also remain close to the learned graph in terms of coverage, but both yield higher Winkler scores. The fully connected graph maintains coverage above 90\%, but produces the largest interval width and the worst Winkler score among the evaluated CKW variants. This suggests that dense global aggregation is less effective than sparse relational structure for regional-scale uncertainty calibration.

The random graph with matched sparsity performs reasonably on both datasets but remains worse than the learned graph in terms of Winkler score. This shows that the benefit of the graph component is not solely due to adding arbitrary neighbor-aggregation features with a fixed sparsity level. Rather, meaningful relational structure improves the calibration-efficiency trade-off. At the same time, the relatively small gap between graph variants suggests that STOIC is reasonably robust to the exact graph construction, likely because temporal residual features and TabPFN calibration also contribute substantially to performance.

Overall, the graph ablation confirms that the learned graph provides the best empirical performance across both EWZ and CKW, while correlation-based graphs remain competitive and random or fully connected alternatives lead to weaker calibration-efficiency trade-offs. These results support the use of learned relational structure in STOIC, but also indicate that the method is not overly sensitive to a single graph specification.

\section{Conclusion}
\label{sec:conclusion}
In this paper, we proposed STOIC, a relational and sequential conformal-inspired uncertainty quantification framework for graph-structured energy time series. By decoupling point forecasting from interval estimation, STOIC provides a post-hoc residual calibration framework that adapts interval widths to spatially and temporally structured residual behavior. Our approach leverages the strengths of spatial-temporal graph neural networks to capture complex system dependencies while utilizing the zero-shot capabilities of tabular foundation models to perform adaptive and in-context calibration of residuals. Experimental results across five diverse benchmarks spanning synthetic simulations and real-world electricity and district heating datasets demonstrate that STOIC consistently achieves favorable empirical coverage--efficiency trade-offs relative to the considered baselines. Across all datasets, STOIC either matches or exceeds the target coverage level while maintaining competitive or substantially narrower prediction intervals, resulting in consistently strong Winkler scores. Specifically, the integration of in-context learning enables adaptive residual calibration with aggregate empirical coverage close to or above the nominal target across the evaluated datasets. The ablation studies further demonstrate that both spatial and temporal features contribute meaningfully to performance, while the TabPFN-based calibration module provides substantial gains relative to the evaluated task-specific alternatives, highlighting the importance of foundation-model-based calibration. Moreover, STOIC remains effective even when only a small calibration dataset is available, demonstrating strong data efficiency and making the framework particularly attractive for practical deployment scenarios where labeled calibration data are limited. These results suggest potential utility for modern sustainable energy systems, while the regime-wise analysis indicates that high-volatility periods remain an important limitation of the current approach.

The comparison with TimesFM further highlights the importance of explicit calibration and domain-aware modeling. Despite being a large-scale pretrained forecasting model, TimesFM fails to achieve the desired coverage levels on real-world energy datasets, whereas STOIC achieves empirical coverage close to or above the nominal target while preserving competitive interval widths. This result demonstrates that large-scale pretraining alone is insufficient for reliable uncertainty quantification and that principled calibration remains essential for obtaining trustworthy prediction intervals in complex energy systems.

Overall, the results establish STOIC as an effective framework for adaptive post-hoc uncertainty quantification in spatial-temporal forecasting problems. The combination of graph-based residual representations and foundation-model-based in-context learning provides a modular post-hoc approach for generating context-aware prediction intervals under spatial dependencies, temporal non-stationarity, and limited calibration data. When finite-sample marginal validity is required under exchangeability, STOIC can be combined with an explicit held-out conformal correction.

Future work can explore integrating a broader range of exogenous variables into the in-context prompt, such as real-time weather alerts, to further improve interval sharpness. Additionally, the present study focuses on one-step-ahead daily forecasting. Extending STOIC to multi-horizon forecasting and investigating how residual dependencies and predictive uncertainty evolve across different forecast horizons represent important directions for future work.

\section*{Acknowledgments}

This work was funded by the Swiss Federal Institute of Metrology (METAS).

\bibliographystyle{tmlr}
\bibliography{main}

\end{document}